\documentclass[pdflatex,sn-mathphys-num]{sn-jnl}% Math and Physical Sciences Numbered Reference Style 
\usepackage{graphicx}%
\usepackage{multirow}%
\usepackage{amsmath,amssymb,amsfonts}%
\usepackage{amsthm}%
\usepackage{mathrsfs}%
\usepackage[title]{appendix}%
\usepackage{xcolor}%
\usepackage{textcomp}%
\usepackage{manyfoot}%
\usepackage{booktabs}%
\usepackage{algorithm}%
\usepackage{algorithmicx}%
\usepackage{algpseudocode}%
\usepackage{listings}%

\usepackage{bm}
\usepackage{subcaption}
\usepackage{todonotes}
\theoremstyle{thmstyleone}%
\newtheorem{theorem}{Theorem}%  meant for continuous numbers
\newtheorem{proposition}[theorem]{Proposition}% 
\newtheorem{corollary}{Corollary}

\theoremstyle{thmstyletwo}%
\newtheorem{example}{Example}%
\newtheorem{remark}{Remark}%

\theoremstyle{thmstylethree}%
\newtheorem{definition}{Definition}%
\newtheorem{property}{Property}
\newtheorem{assumption}{Assumption}

\graphicspath{{../../.}{../.}{../../images}}

\DeclareMathOperator{\R}{\mathbb{R}}
\DeclareMathOperator{\N}{\mathbb{N}}
\DeclareMathOperator{\E}{\mathbb{E}}

\DeclareMathOperator{\trace}{trace}
\DeclareMathOperator{\Var}{Var}

\DeclareMathOperator{\KL}{D_\mathrm{KL}}

\DeclareMathOperator{\Hell}{D_\mathrm{Hell}}

\DeclareMathOperator{\AlphaDiv}{D_{\alpha}}

\DeclareMathOperator*{\argmin}{arg\,min}

\renewcommand{\d}{\mathrm{d}}
\newcommand{\Xc}{\mathcal{X}}
\newcommand{\Yc}{\mathcal{Y}}

\DeclareMathOperator{\rhoref}{\rho_{\mathrm{ref}}}
\newcommand{\pitar}{\pi}
\newcommand{\pibridge}[1]{\pi^{(#1)}}
\newcommand{\pitartilde}{\widetilde{\pi}}
\newcommand{\pibridgetilde}[1]{\widetilde{\pi}^{(#1)}}

\newcommand{\pipullbacktilde}[1]{\rho^{(#1)}}

\begin{document}
	
	\title[Sequential transport maps using SoS density estimation and $\alpha$-divergences]{Sequential transport maps using SoS density estimation and $\alpha$-divergences}

	\author*[1]{\fnm{Benjamin} \sur{Zanger}}\email{benjamin.zanger@inria.fr}
	\equalcont{These authors contributed equally to this work.}
	
	\author[1]{\fnm{Olivier} \sur{Zahm}}\email{olivier.zahm@inria.fr}
	\equalcont{These authors contributed equally to this work.}
	
	\author[2]{\fnm{Tiangang} \sur{Cui}}\email{tiangang.cui@sydney.edu.au}
	
	\author[1]{\fnm{Martin} \sur{Schreiber}}\email{martin.schreiber@univ-grenoble-alpes.fr}
	
	\affil*[1]{\orgname{Université Grenoble Alpes, Inria, CNRS, Grenoble INP, LJK}, \orgaddress{\street{150 Pl. du Torrent}, \city{Saint-Martin-d'Hères}, \postcode{38400}, \country{France}}}
	
	\affil[2]{\orgdiv{School of Mathematics and Statistics}, \orgname{University of Sydney}, \orgaddress{\street{NSW 2006}, \country{Australia}}}

	
	\abstract{Transport-based density estimation methods are receiving growing interest because of their ability to efficiently generate samples from the approximated density.
		We further investigate the sequential transport maps framework proposed from~\cite{cui_scalable_2023, cui_self-reinforced_2023}, which builds on a sequence of composed Knothe--Rosenblatt (KR) maps.
		Each of those maps are built by first estimating an intermediate density of moderate complexity, and then by computing the exact KR map from a reference density to the precomputed approximate density.
		In our work, we explore the use of Sum-of-Squares (SoS) densities and $\alpha$-divergences for approximating the intermediate densities.
		Combining SoS densities with $\alpha$-divergence interestingly yields \emph{convex} optimization problems which can be efficiently solved using semidefinite programming.
		The main advantage of $\alpha$-divergences is to enable working with \emph{unnormalized} densities, which provides benefits both numerically and theoretically.
		In particular, we provide a new convergence analyses of the sequential transport maps based on information geometric properties of $\alpha$-divergences.
		The choice of intermediate densities is also crucial for the efficiency of the method. While tempered (or annealed) densities are the state-of-the-art, we introduce diffusion-based intermediate densities which permits to approximate densities known from samples only. 
		Such intermediate densities are well-established in machine learning for generative modeling.
		Finally we propose low-dimensional maps (or lazy maps) for dealing with high-dimensional problems and numerically demonstrate our methods on Bayesian inference problems and unsupervised learning tasks.}

	
	
	
	\maketitle

	\section{Introduction}

	Density estimation is a fundamental problem in data sciences.
	Transport-based methods are receiving growing interest because of their ability to sample easily from the approximated density \cite{marzouk_introduction_2016,grenioux2023sampling,parno2018transport,rezende2015variational,maurais2024sampling}.
	These methods aim at building a deterministic diffeomorphism $\mathcal{T}$, also called a transport map, which pushes forward an arbitrary reference probability density $\rhoref$ to a given target probability density $\pitar$ to be approximated.
% 	In the context of \emph{e.g} Bayesian inference or molecular dynamic, the density $\pitar$ can be evaluated up to a normalizing constant.
	This pushforward density, denoted by $\mathcal{T}_\sharp \rhoref$, is the density of the random vector $\mathcal{T}(\bm \xi)$, where $\bm \xi\sim\rhoref$.
	Variational methods consist in solving a problem of the form
	\begin{equation}\label{eq:variational_TM}
		\min_{\mathcal{T} \in\mathcal{M}} \mathrm{D}(\pitar ||\mathcal{T}_\sharp \rhoref),
	\end{equation}
	where the statistical divergence $\mathrm{D}(\cdot || \cdot)$ and the set of diffeomorphisms $\mathcal{M}$ are typically chosen so that problem \eqref{eq:variational_TM} is a tractable problem.
	Typically, the Kullback--Leibler (KL) divergence is employed for $\mathrm{D}(\cdot\ || \cdot)$ when only samples from $\pitar$ are available (learning from data), while the reverse KL divergence is utilized when $\pitar$ can be evaluated up to a normalizing constant (learning from unnormalized density). Minimizing the reverse KL divergence, however, fails to capture multi-modal densities $\pitar$, a well known problem for such zero-forcing divergences, see e.g.~\cite{cichocki_families_2010,noe2019boltzmann,felardos_designing_2023}.
	Regarding the set $\mathcal{M}$, the first cornerstone for handling high-dimensional diffeomorphic maps is the monotone triangular map, where each $i$-th map component depends on the first $i$ variables only and is monotone in the $i$-th variable \cite{baptista2023representation,jaini2019sum,dinh_density_2017}. As to increase the approximation power, the second cornerstone is to compose multiple monotone triangular map.
	This is the basic idea of many map parametrization used in the Normalizing flows literature \cite{papamakarios_normalizing_2021,felardos_designing_2023,rezende2015variational} to cite just a few.
	Let us emphasis that, in principle, there exists infinitely many maps $\mathcal{T}$ which satisfy $\mathcal{T}_\sharp \rhoref = \pitar$. 
	This is contrarily to the Monge optimal transport problem which seeks the map $\mathcal{T}$ that minimizes some transport-related loss $\int(\bm x - \mathcal{T}(\bm x))^2\d\rhoref(\bm x)$ under the constraint $\mathcal{T}_\sharp \rhoref =\pitar$, see \cite{villani2009optimal,peyre2019computational}.
	In contrast, the goal of \eqref{eq:variational_TM} is to build \emph{a tractable} map $\mathcal{T}\in\mathcal{M}$ such that $\mathcal{T}_\sharp \rhoref \approx \pitar$.
	In other terms, problem \eqref{eq:variational_TM} is not concerned with the paths by which $\mathcal{T}$ transports $\rhoref$ to $\pitar$; it solely focuses on the resulting distance between $\mathcal{T}_\sharp \rhoref$ and $\pitar$.

	\begin{figure}
		\includegraphics[width=\textwidth]{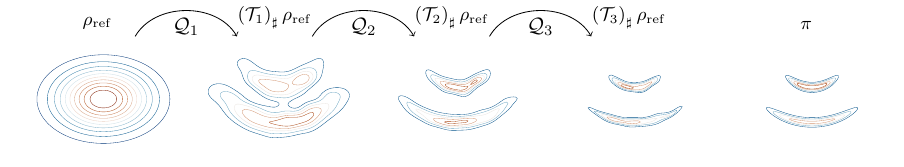}
		\caption{Visualization of the approximation of a bimodal density $\pitar$ (right) using $L=3$ intermediate tempered  densities estimated using SoS \eqref{eq:approx_class_SoS} and a Gaussian reference density $\rhoref$.}
		\label{fig:approx_double_banana_intro}
	\end{figure}
	
	An emerging strategy for solving \eqref{eq:variational_TM} is to first approximate $\pitar$ with an approximate density $\widetilde\pi$ and then to compute a map $\mathcal{T}$ which exactly pushes forward $\rhoref$ to $\widetilde\pi$.
	Among the infinitely many maps $\mathcal{T}$ which satisfy $\mathcal{T}_\sharp\rhoref = \widetilde\pi$, the Knothe--Rosenblatt (KR) map is rather simple to evaluate since it requires only computing the cumulative distribution functions (CDFs) of the conditional marginals of $\widetilde\pi$.
	In the seminal work \cite{dolgov2020approximation}, $\widetilde\pi$ is built in the tensor-train format, an approximation class which permits to efficiently compute the KR map, see also \cite{cui_deep_2021,cui_scalable_2023,cui2024deep}. Polynomial approximation methods have also been used recently in \cite{cui_self-reinforced_2023,westermann_measure_2023} for building transport maps.
	In all these papers, the statistical divergence $\mathrm{D}(\cdot || \cdot)$ is chosen to be the Hellinger distance, which is the $L^2$ distance between square root densities. Conveniently, this distance permits building approximations $\widetilde\pi$ to $\pitar$ using standard $L^2$ function approximation techniques, like polynomial interpolation, least-squares or tensor methods.
	In general, however, the variational problem $\min_{\widetilde\pi}\mathrm{D}(\pitar||\widetilde\pi)$ is difficult to solve when $\pitar$ is multimodal or when it concentrates on a low-dimensional manifold.
	The solution proposed in \cite{cui_deep_2021} consists in introducing an arbitrary sequence of bridging densities
	\begin{equation}\label{eq:bridging_intro}
		\pibridge{1},\pibridge{2},\hdots,\pibridge{L} = \pitar,
	\end{equation}
	with increasing complexity. 
	Similar to Sequential Monte Carlo Samplers \cite{del_moral_sequential_2006}, this sequence allows for breaking down the challenging approximation problem \eqref{eq:variational_TM} into a sequence of intermediate problems of more manageable complexity.
	A similar approach was also proposed in~\cite{eigel_less_2024} to guide the dynamics in Langevin based samplers for faster convergence and better capturing multimodal densities.
	A classical choice of bridging densities are tempered densities
	$\pibridge{\ell}(\bm x) \propto \pitar(\bm x)^{\beta_\ell} \pi_0(\bm x)^{1 - \beta_\ell}$ with parameters $0\leq\beta_0\leq\hdots\leq\beta_L=1$ and arbitrary density $\pi_0$~\cite{cui_deep_2021,cui_self-reinforced_2023,cui_scalable_2023}, although other relevant choices are possible depending on the application \cite{cui2024deep, del_moral_sequential_2006}.
	The sequential strategy consists in building $L$ transport maps $\mathcal{Q}_{1},\hdots,\mathcal{Q}_{L}$ one after the other by solving
	\begin{equation}\label{eq:variational_density_problem}
		\min_{ \mathcal{Q}_{\ell} \in\mathcal{M} } \mathrm{D}( \pibridge{\ell} || ( \mathcal{T}_{\ell-1} \circ \mathcal{Q}_{\ell} )_\sharp \rhoref) ,
		\quad\text{where}\quad
		\mathcal{T}_{\ell-1} = \mathcal{Q}_{1}\circ\hdots\circ \mathcal{Q}_{\ell-1}.%
	\end{equation}
	An illustration of such an sequential approximation using $L=3$ steps is depicted in Figure~\ref{fig:approx_double_banana_intro}.
	For suitable statistical distances, so that $\mathrm{D}(\pi || \mathcal{T}_\sharp \rho) = \mathrm{D}(\mathcal{T}^\sharp \pi || \rho)$, these problems are equivalent to estimating the pullback density $(\mathcal{T}_{\ell-1})^\sharp\pibridge{\ell}$ with an intermediate approximation $\pipullbacktilde{\ell}=(\mathcal{Q}_{\ell} )_\sharp \rhoref$.
	Again, this can be done by first building the approximate density $\pipullbacktilde{\ell}$ and then by extracting the KR map $\mathcal{Q}_{\ell}$ which pushes forward $\rhoref$ to the precomputed $\pipullbacktilde{\ell}$.

	Our first contribution is to
	employ Sum-of-Squares (SoS) densities to approximate the intermediate densities $\rho^{(\ell)}$ using $\alpha$-divergences $\AlphaDiv(\cdot||\cdot)$. Hence, we sequentially solve the variational density approximation problem
	\begin{equation}\label{eq:SoS_density_approximation}
		\min_{A_\ell \succeq 0} \AlphaDiv( \pitar|| (\mathcal{T}_{\ell-1})_\sharp \rho^{(\ell)}),
	\end{equation}
	where
	\begin{equation}\label{eq:approx_class_SoS}
		\rho^{(\ell)}(\bm x) = \left( \Phi(\bm x)^\top A_\ell \Phi(\bm x) \right) \rhoref(\bm x),
	\end{equation}
	for some arbitrary orthonormal basis function $\Phi$ in $L^2(\rhoref)$.
	Here, the positivity of the matrix $A_\ell \succeq 0$ ensures the density $\rho^{(\ell)}$ to be positive.
	Since the $\alpha$-divergence is defined for general \emph{unnormalized} densities, it is not necessary to know the normalizing constant of $\pi$ beforehand nor to impose $\int  \rho^{(\ell)}(\bm x)\d\bm x= 1$ when solving \eqref{eq:SoS_density_approximation} ($\rho^{(\ell)}$ can be normalized in a second step).
	Note that SoS densities~\eqref{eq:approx_class_SoS} is a generalization of the polynomial squared method used in \cite{cui_self-reinforced_2023,westermann_measure_2023} where $\pipullbacktilde{\ell} (\bm x) \propto (  v_\ell^\top \Phi(\bm x) )^2 \rhoref(\bm x)$.
	SoS densities permit to efficiently compute the KR map $\mathcal{Q}_{\ell}$ such that $(\mathcal{Q}_{\ell})_\sharp\rhoref = \pipullbacktilde{\ell}$, which is a crucial property for the proposed methodology.
	Let us mention that SoS functions have also been used in \cite{jaini2019sum,akata_riemannian_2021} for parametrizing the map $\mathcal{T}$ directly in order to ensure the monotonicity of each map component.
	The proposed $\alpha$-divergences $\AlphaDiv(\cdot || \cdot)$ with parameter $ \alpha \in \R $ include the Hellinger distance and KL-divergence used in previous works.
	Most importantly, the use of $\alpha$-divergence for performing SoS density estimation as in~\eqref{eq:SoS_density_approximation} results in a \emph{convex} optimization problems which can be efficiently solved using off-the-shelf toolboxes.

	Our second contribution is to extend the methodology to the scenario where \emph{only samples} $\bm X^{(1)},\hdots,\bm X^{(N)}$ from $\pitar$ are available, as opposed to \emph{e.g.} \cite{cui_self-reinforced_2023,eigel_less_2024,westermann_measure_2023,cui_deep_2021,cui_scalable_2023} which rely on \emph{point-evaluations} of the target density $\pitar$.
	When learning from samples, however,
	the tempered bridiging densities $\pibridge{\ell}(\bm x) \propto \pitar(\bm x)^{\beta_\ell} \pi_0(\bm x)^{1-\beta_\ell}$ cannot be employed because $\pitar(\bm x)$ is not accessible. We rather consider diffusion-based bridging densities
	\begin{equation}\label{eq:bridge_diffusion}
		\pibridge{\ell}(\bm x) = \int \kappa_{t_\ell}(\bm x,\bm y) \pitar(\bm y) \d \bm y ,
	\end{equation}
	with time parameters $t_1 \geq\hdots\geq t_L = 0$, where $\kappa_{t}(\cdot,\cdot)$ is the kernel of a diffusion process such as the Ornstein-Uhlenbeck process.
	By construction, samples $\bm X_\ell^{(i)}\sim \pibridge{\ell}$ are readily obtained simply by rescalling and adding noise to the samples $\bm X^{(i)}\sim \pitar$.
	This idea is at the root of diffusion models \cite{sohl2015deep,song2021scorebased}.
	The use of diffusion-based bridging densities as in \eqref{eq:bridge_diffusion} is also receiving growing attention to accelerate MCMC sampling, see \cite{grenioux2024stochastic,akhound2024iterated}.

	Our third contribution is a novel convergence analysis using the geometric properties of $\alpha$-divergences.
	This analysis unifies and extends previous analyses proposed in \cite{cui_self-reinforced_2023,westermann_measure_2023,cui_deep_2021,cui_scalable_2023} and, more interestingly, it guides the choice of briding densities. In particular, we show that a smart choice of $\beta_\ell$ for tempered densities or of $t_\ell$ for diffusion-based densities yield a convergence rate of $\mathcal{O}(1/L^2)$ with respect to the number of layer $L$.

	The rest of the paper is as follows.
	We introduce $\alpha$-divergences in Section~\ref{sec:AlphaDiv} and SoS functions in Section~\ref{sec:SoS_function_class}.
	We then explain the sequential framework using bridging densities and provide convergence analysis of sequential transport maps in Section~\ref{sec:SequentialTM}.
	Practical considerations for high-dimensionality and an algorithm for working with datasets is presented in Section~\ref{sec:highdim}.
	Finally, in Section~\ref{sec:Numerical_examples}, numerical examples demonstrate the feasibility of the proposed methods.
	
	\section{Variational density estimation using $\alpha$-divergence}\label{sec:AlphaDiv}

	We propose to use $\alpha$-divergences~\cite{cichocki_families_2010} for the variational density estimation. Using the definition of $\alpha$-divergences from~\cite{zhu_bayesian_1995}, for a given $\alpha\in\R$, the $\alpha$-divergence between two \emph{unnormalized} densities $f$ and $g$ defined on $\Xc\subseteq\R^d$ (meaning integrable and positive functions) reads
	\begin{equation}\label{eq:def_alpha_divergences}
		\AlphaDiv\left(f || g\right) = \int_{\Xc} \phi_{\alpha}\left(\frac{f(\bm x)}{g(\bm x)}\right) g(\bm x) \d\bm x, \quad \text{with } \phi_{\alpha}(t) = \begin{cases}
			\frac{t^{\alpha} - 1}{\alpha (\alpha -1)} - \frac{t - 1}{\alpha - 1}  &\alpha\notin\{0, 1\} \\
			t \ln(t) - t + 1  &\alpha = 1 \\
			- \ln(t) + t - 1  &\alpha = 0.
		\end{cases}
	\end{equation}
	Denoting by $\pi_f(\bm x)\propto f(\bm x)$ and $\pi_g(\bm x)\propto g(\bm x)$ the probability densities obtained by normalizing respectively $f$ and $g$, the $\alpha$-divergences for \emph{normalized} densities simplify to
	\begin{equation}\label{eq:def_alpha_divergences_normalized}
		\mathrm{D}_{\alpha}^{n}(\pi_f || \pi_g) = \int \phi^n_{\alpha}\left(\frac{\pi_f}{\pi_g}\right) \d \pi_f, \quad \text{with } \phi_{\alpha}^{n}(t) = \begin{cases}
			\frac{t^{\alpha} - 1}{\alpha (\alpha -1)} &\alpha\notin\{0, 1\} \\
			t \ln(t)  &\alpha = 1 \\
			- \ln(t) &\alpha = 0.
		\end{cases}
	\end{equation}
	Notice that for all $\alpha\in\R$, the function $\phi_\alpha:\R_{\geq0}\rightarrow\R_{\geq0}$ is positive, convex $\phi_\alpha''(t)>0$, and is minimial at $t=1$, meaning $\phi_\alpha'(1)=0$, while this is not the case for $\phi_{\alpha}^{n}$.
	The choices $\alpha\in\{0,1/2,1,2\}$ yield respectively to
	\begin{align}\label{eq:alpha_div_relation_to_other_div}
		\mathrm{D}_{1}\left( \pi_f || \pi_g\right) &= \KL\left( \pi_f || \pi_g \right), 	&
		\mathrm{D}_{0}\left( \pi_f || \pi_g\right) &= \KL\left( \pi_g || \pi_f\right),\\
		\mathrm{D}_{1/2}\left( \pi_f || \pi_g\right) &= 4 \Hell( \pi_f || \pi_g )^2,	 	&
		\mathrm{D}_{2}\left( \pi_f || \pi_g\right) &= \frac{1}{2}\chi^2(\pi_f || \pi_g),
	\end{align}
	where $\KL\left( \pi_f || \pi_g \right) = \int \ln(\frac{\pi_f}{\pi_g} )\d\pi_f $ is the Kullback--Leibler divergence, $\Hell( \pi_f || \pi_g ) = (\frac{1}{2}\int ( \sqrt{\pi_f}-\sqrt{\pi_g} )^2\d \bm x )^{1/2}$ the Hellinger distance, and $\chi^2(\pi_f || \pi_g) = \int( \frac{\pi_f}{\pi_g}-1 )^2\d \pi_g  $ the chi-square divergence.
	
	One practical advantage of using $\alpha$-divergences with \emph{unnormalized} densities is that there is no need to enforce the approximate density to integrate to one while learning it --- the resulting approximate density is being normalized in a second step. 
	In practice, one does not require knowing the normalizing constant of the target density, which is convenient \emph{e.g.} for Bayesian inverse problems where the posterior density is known up to a multiplicative constant.
	In addition, it is argued in \cite{felardos_designing_2023} that a naive discretization of the normalized KL divergence $\KL\left( \pi_f || \pi_g \right) = \mathrm{D}_{1}^{n}(\pi_f || \pi_g) = \int \ln(\frac{\pi_f}{\pi_g} )\d\pi_f $ might yield unstable estimate, whereas the discretization of $\mathrm{D}_{1}(f||g)$ as in~\eqref{eq:def_alpha_divergences} does not suffer from such instability because of the affine term $-t+1$, see the detailled discussion in~\cite[Appendix A]{nielsen_elementary_2020}.
	In Figure~\ref{fig:alpha_div_probabilities_and_measure} we visualize the importance of the affine term in the definition of $\phi_\alpha(t)$.
	Denoting by $Z_f=\int f( \bm x) \d \bm x$ and $Z_g=\int g( \bm x) \d \bm x$ the normalizing constant of $f$ and $g$, we show in Appendix~\ref{proof:measure_prob_bound_alpha_div} that
	\begin{align}\label{eq:AlphaDivNormalized}
		\AlphaDiv\left( f || g \right) = \frac{Z_{f}^{\alpha}}{Z_{g}^{\alpha - 1}} \mathrm{D}_{\alpha}^{n}\left( \pi_f || \pi_g \right) + Z_{g} \phi_{\alpha}\left(\frac{Z_f}{Z_g}\right) ,
	\end{align}
	holds for any unnormalized densities $f,g$.
	This relation suggests that for a density $f$, minimizing $g\mapsto \AlphaDiv\left( f || g \right)$ permits to control both $\AlphaDiv( \pi_f || \pi_g )$ and $\phi_{\alpha}(\tfrac{Z_f}{Z_g})$ so that, after normalizing $g$, the probability density $\pi_g$ yields a controlled approximation to $\pi_f$.
	\begin{figure}
		\begin{subfigure}{0.45\textwidth}
			\includegraphics[width=\textwidth]{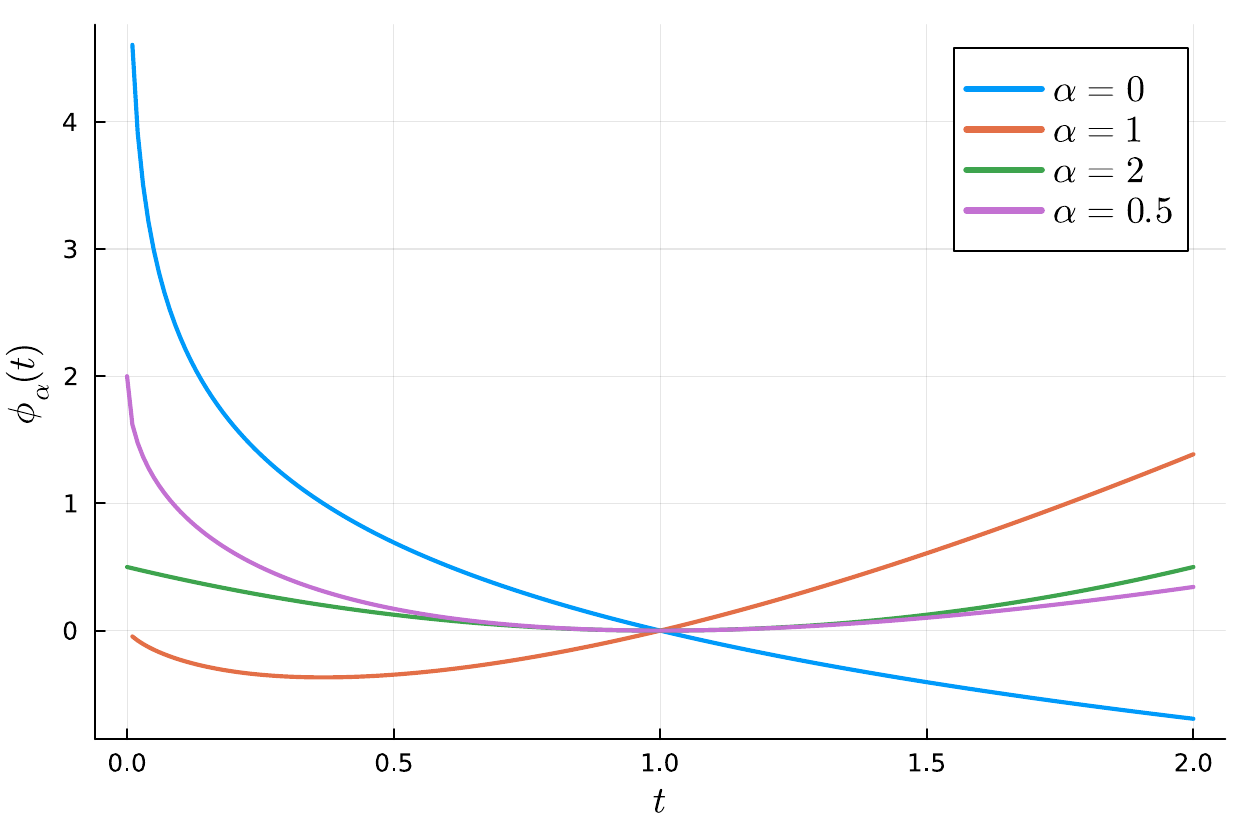}
			\caption{$\phi_{\alpha}^n$ for probability densities}
		\end{subfigure}
		\begin{subfigure}{0.45\textwidth}
			\includegraphics[width=\textwidth]{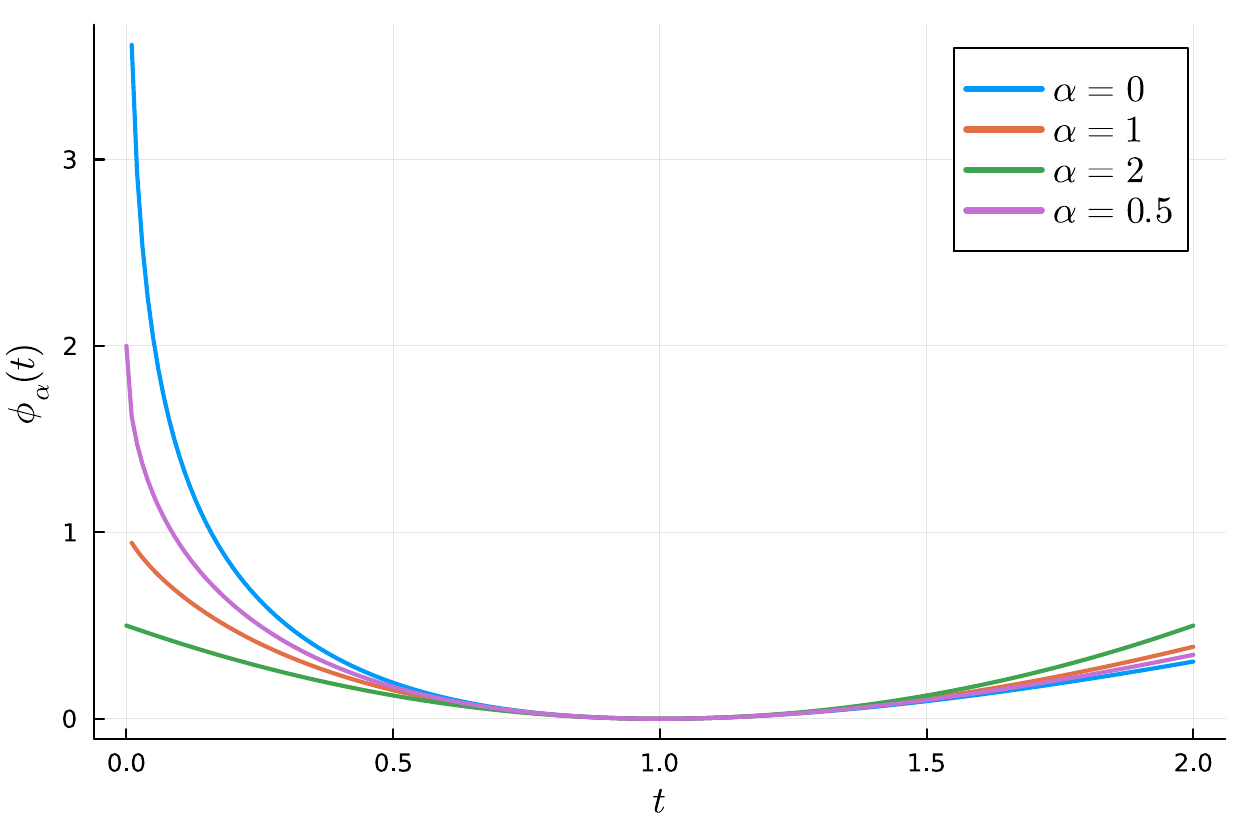}
			\caption{$\phi_{\alpha}$ for unnormalized densities}
		\end{subfigure}
		\caption{Left: functions $\phi_{\alpha}^{n}(t)$ as in Eq.~\eqref{eq:def_alpha_divergences_normalized} associated with $\alpha$-divergences for \emph{normalized} densities. Right: functions $\phi_{\alpha}(t)=\phi_{\alpha}^{n}(t)-\frac{t-1}{\alpha-1}$ as in \eqref{eq:def_alpha_divergences} associated with $\alpha$-divergences for \emph{unnormalized} densities.
		Contrarily to $\phi_{\alpha}^{n}$, the affine term $-\frac{t-1}{\alpha-1}$ preserves the convexity of $\phi_{\alpha}$ while ensuring $\phi_{\alpha}$ to admit a minimum at $t=1$.
		}\label{fig:alpha_div_probabilities_and_measure}
	\end{figure}
	According to \cite{minka_divergence_2005}, minimizing $g\mapsto \AlphaDiv\left(f || g\right)$ with $\alpha \leq 0$ forces $g(\bm x) \ll 1$ in regions where $f(\bm x) \ll 1$, thus avoiding \textit{false positives} (zero forcing property).
	Reciprocally, minimizing $g\mapsto \AlphaDiv\left(f || g\right)$ with $\alpha \gg 1$ forces $g(\bm x) \gg 1$ in regions where $f(\bm x) \gg 1$, which avoids \textit{false negatives} (zero avoiding property).
	This is coherent with the duality property
	\begin{align}\label{eq:DualityProperty}
		\AlphaDiv\left(f || g\right) = \mathrm{D}_{1-\alpha}\left(g || f\right),
	\end{align}
	which holds for any $\alpha\in\R$, see e.g.~\cite{amari_information_2016}.
	Another fundamental property of $\alpha$-divergences is the stability under transportation by a diffeomorphism $\mathcal{T}:\Xc\rightarrow\Xc$, meaning
	\begin{align}\label{eq:TransportProperty}
		\AlphaDiv\left( f || \mathcal{T}_\sharp g \right) &= \AlphaDiv\left( \mathcal{T}^\sharp f || g\right).
	\end{align}
	As detailed later in Section \ref{sec:SequentialTM}, this property permits us to reformulate problem \eqref{eq:variational_density_problem} as an approximation of $\mathcal{T}_{\ell-1}^\sharp \pibridge{\ell}$, the pullback of the $\ell$-th bridging density via the map $\mathcal{T}_{\ell-1}$.
	For the sake of completeness, the derivation of Equation \eqref{eq:TransportProperty} is given in Appendix~\ref{proof:pushforward_pullback_equiv_in_alpha_div}.
	
	In practice, we distinguish the two cases when either the target probability density $\pitar$ can be arbitrarily evaluated up to a multiplicative constant (learning from density), or when the target probability density is only known via samples from $\pitar$ (learning from data).
	The first case typically corresponds to Bayesian inverse problems, where the goal is to sample from the posterior density
	\begin{equation}\label{eq:deff}
		\pitar(\bm x)\propto f(\bm x) := \mathcal{L}(\bm x) \pi_0(\bm x).
	\end{equation}
	Here, $\pi_0$ denotes the prior density and $\mathcal{L}$ the likelihood function of a given data set conditioned on $\bm x$, typically $\mathcal{L}(\bm x)=\exp(-\frac{1}{2}\|y-u(\bm x)\|^2)$ for some forward model $u:\Xc\rightarrow\R^q$ and some observed data set $y\in\R^q$. 
	A different scenario is rare event estimation where $\mathcal{L}(\bm x) = 1_{u(\bm x)\geq p}$ is the indicator function (of level $p$) of the event $u(\bm X)\geq p $, $ \bm X\sim\pi_0$, which we need to compute the probability of.

	In the second case (learning from data), setting $\alpha=1$ yields the KL divergence $\mathrm{D}_1(\pi||g) = \int \ln(\pi)\d\pi -1-\mathcal{J}_\text{KL}(g)$, where $\mathcal{J}_\text{KL}(g)=\int\ln(g)\d\pi + \int g\d \bm x$.
	To estimate $\mathrm{D}_\alpha(f||g)$, one can use a Monte-Carlo estimation of the form of
	\begin{equation}\label{eq:DalphaHat}
		\widehat{\mathrm{D}}_\alpha(f||g) = \frac{1}{N} \sum_{i=1}^N
		\phi_{\alpha}\left(\frac{f(\bm X^{(i)})}{g(\bm X^{(i)})}\right) \frac{g(\bm X^{(i)})}{\rhoref(\bm X^{(i)})} ,
	\end{equation}
	where $\bm X^{(1)},\hdots, \bm X^{(N)}\sim\rhoref$ are independent samples drawn from a given reference measure $\rhoref$.
	On the other hand, in the case of given data, we assume that the data $\bm X^{(1)},\hdots, \bm X^{(N)}$ is independently distributed according to $\pitar$ and
	\begin{equation}\label{eq:KLHat}
		\widehat{\mathcal{J}}_\text{KL}(g) = - \frac{1}{N} \sum_{i=1}^N  \ln\left(g(\bm X^{(i)})\right) +  \int g\d \bm x.
	\end{equation}
	We discuss the choice of $\rhoref$ in Section~\ref{sec:Rosenblatt}.
	
	\begin{property}[Convexity of $\alpha$-divergence]
		The function $t\mapsto  t \phi_\alpha( u /t)$ is convex for any $\alpha\in\R$ and $u\in\R_{\geq0}$.
		As a consequence, the functions $g\mapsto \AlphaDiv(f||g)$, $g\mapsto \widehat{\mathrm{D}}_\alpha(f||g)$, and $g \mapsto \widehat{\mathcal{J}}_\text{KL}(g)$ are also convex.
	\end{property}

	\section{Sum-of-Squares densities}\label{sec:SoS_function_class}
	Sum-of-Squares (SoS) functions $g_A$ are functions of the type
	\begin{align}\label{eq:SoS_function}
		g_A(\bm x) = \left( \Phi(\bm x)^\top A \Phi(\bm x) \right) \rhoref(\bm x),
	\end{align}
	parameterized by a positive semidefinite matrix $A\succeq 0$ and where $\Phi(\bm x)=(\phi_1(\bm x),\hdots,\phi_m(\bm x))$ is a vector consisting of functions $\phi_i \in L^2_{\rhoref}(\Xc)$, where $\Xc\subseteq\R^d$ is the support of a reference density $\rhoref$.
	By construction, $g_A: \Xc \rightarrow \R_{\geq 0}$ is non negative and integrable.
	Indeed, denoting by $A=\sum_{i=1}^m \lambda_i u_iu_i^\top$ the eigenvalue decomposition of $A$ with $\lambda_i\geq0$ and $u_i\in\R^m$, the function $\bm x\mapsto \Phi(\bm x)^\top A \Phi(\bm x) = \sum_{i=1}^m\lambda_i ( u_i^\top \Phi(\bm x) )^2$ is the sum of $m$ squared functions, which explains the terminology SoS.
	Note that $g_A$ does not necessarily integrate to one. 
	We call SoS functions which integrate to one SoS \textit{densities} and introduce them in Section~\ref{sec:Rosenblatt}.
	Since the parametrization $A\mapsto g_A$ is linear, the function $A\mapsto \widehat{\mathrm{D}}_\alpha(f||g_A)$ as in \eqref{eq:DalphaHat} remains convex, as already noticed in~\cite{marteau-ferey_non-parametric_2020}.
	The resulting problem 
	\begin{equation}
		\argmin_{A\succeq0} \widehat{\mathrm{D}}_\alpha(f||g_A)
	\end{equation}
	is convex and can be solved by using semidefinite programming (SDP).
	Toolboxes to solve such problems computationally are for example \textit{JuMP.jl}~\cite{Lubin2023} and \textit{CVX}~\cite{cvx}.

	\begin{remark}
		While $A \succeq 0$ is sufficient to ensure the positivity of $g_A$, this condition can be relaxed depending on $\Xc$ and $\Phi$. For instance, for polynomial basis $\Phi$ and a semi-algebraic set $\Xc$, the condition $A \succeq 0$ can be replaced with a weaker condition e.g. based on Putinar's Positivstellensatz, see~\cite{lasserre_moments_2010, putinar_positive_1993}. This is, however, not considered in the present paper.
	\end{remark}

	\subsection{Orthonormal tensorized basis for Sum-of-Squares}\label{subsec:orthonormal_polynomial_feature_maps}

	Assuming  $\Xc = \Xc_1 \times \dots \times \Xc_d$ is a product space and $\rhoref(\bm x) = \rhoref_1(x_1)\hdots\rhoref_d(x_d)$ a product measure, we can construct $\Phi: \Xc \rightarrow \R^m$ by tensorizing univariate orthonormal functions as follow. Let $\left\{\phi^{k}_i\right\}_{i=1}^{\infty}$ be a set of orthonormal functions on $L^2_{\rhoref_k}(\Xc_k)$, $1\leq k \leq d$ and, given a multi-index $\alpha\in\N^d$, consider the following tensorization
	$$
	\phi_\alpha(\bm x) = \phi_{\alpha_1}^1(x_1) \hdots \phi_{\alpha_d}^d(x_d) .
	$$
	Then $\{\phi_\alpha\}_{\alpha\in\N^d}$ forms an orthonormal basis of $L^2_{\rhoref}(\Xc)$.
	In our work, we use multivariate polynomials to define $\phi_\alpha$ although, in principle, any set of orthonormal functions on $L^2_{\mu_k}(\Xc_k)$ can be used (trigonometric functions, wavelets,...).
	In the following, we give examples of orthogonal bases, including Legendre polynomial on $[-1,1]$, Hermite polynomials on $\R$, and transformed Legendre polynomials on generic domains.
	\begin{example}[Orthogonal polynomials]\label{example:legendre_feature_map}
		A property of any orthogonal polynomials basis $\{P_n\}_{n\geq0}$ in $L^2_{\rhoref}$ is that they can be computed efficiently via a recurrence relation, given by 
		\begin{align}
			P_n(x) = \left(a_n x + b_n\right) P_{n-1}(x) + c_n P_{n-2}(x),
		\end{align}
		where the series $a_n, b_n$, and $c_n$ depend on $\rhoref$~\cite{koornwinder_orthogonal_2013}.
		For the uniform density $\rhoref = \mathcal{U}([-1,1])$ on the bounded domain $[-1,1]$, we obtain the Legendre polynomials with
		\begin{align}
			a_n = \frac{2n+1}{n+1} \qquad b_n = 0 \qquad c_n = \frac{n}{n+1}.
		\end{align}
		The normalized Legendre polynomials $L_n$ are obtained by
		\begin{align}
			L_n(x) = \sqrt{n + \frac{1}{2}} P_n(x).
		\end{align}
		For the Gaussian measure $\rhoref=\mathcal{N}(0,1)$, on the unbounded domain $\R$, we obtain the probabilist's Hermite polynomials $H_n(x)$ with the recurrence relation
		\begin{align}
			a_{n} = 1 \qquad b_n = 0 \qquad c_n = -n. 
		\end{align}
		While Hermite polynomials could be directly used on indefinite domains, for numerical reasons, we prefer to use transformed Legendre polynomials, which we explain in the next example.
	\end{example}
	\begin{example}[Transformed Legendre polynomials~\cite{shen_approximations_2014}]\label{exam:transformed_legendre}
		Another way to obtain an orthogonal univariate basis of $L^2_{\rhoref}(\mathcal{X})$ is to consider Legendre polynomials $L_n$ on $L^2_{\mu}([-1,1])$ with $\mu=\mathcal{U}([-1,1])$ and the map $\mathcal{R}$ such that $\mathcal{R}_{\sharp}\mu = \rhoref$.
		From the normalized Legendre polynomials $\{L_i\}_{i\geq 0}$ we can create functions $\{\phi_i\}_{i\geq 0}$ given by
		\begin{align}
			\phi_i(x) = L_i( \mathcal{R}^{-1}(x) ),
		\end{align}
		which are orthonormal in $L^2_{\rhoref}(\Xc)$ since $\int \phi_i(x) \phi_j(x) \rhoref(x) \d x= \int L_i(x') L_j(x') \mu(x') \d x' = \delta_{i,j}$.
		We show two transformations from~\cite{cui_self-reinforced_2023} and~\cite{shen_approximations_2014} to transform Legendre polynomials to orthogonal functions on $\R$ in Table~\ref{table:transforms_OMF}.
		\begin{table}[htb]
			\caption{Two transforms from~\cite{cui_self-reinforced_2023} to map Legendre polynomials from the domain $[-1, 1]$ to $\R$.}\label{table:transforms_OMF}
			\centering
			\begin{tabular}{c|c|c|c|c}
				\hline
				& $\mathcal{R}(x)$ & $\left|\nabla \mathcal{R}(x)\right|$ & $\mathcal{R}^{-1}(x)$ & $ \left|\nabla \mathcal{R}^{-1}(x)\right| $ \\
				\hline
				logarithmic & $\tanh(x)$ & $1- \tanh(x)^2$ & $\frac{1}{2} \ln\left(\frac{1+x}{1-x}\right)$ & $\frac{1}{1-x^2}$\\
				\hline
				algebraic &  $\frac{x}{\sqrt{1+x^2}}$ & $\frac{1}{ \left(1+x^2\right)^{3/2}}$ & $\frac{x}{\sqrt{1 - x^2}}$ & $\frac{1}{\left(1-x^2\right)^{3/2}}$\\
				\hline
			\end{tabular}
		\end{table}
	\end{example}
	In practice, we work with a subset of indices $\mathcal{K}\subset\N^d$ with finite cardinality $|\mathcal{K}|=m$.
	In our implementation we use
	\begin{align}\label{eq:setK}
		\mathcal{K} = \left\{ \alpha \in \N^{d} : \sum_{i=1}^d \alpha_i\leq p \right\} ,
		\qquad
		m = \binom{p+d}{d} ,
	\end{align}
	which yields the set of polynomials of total degree bounded by $p$.
	In order to vectorize  $\{\phi_\alpha(\bm x)\}_{\alpha\in\mathcal{K}}$ in a vector $\Phi(\bm x)\in\R^m$, we consider a bijective map
	$$
	\sigma:\mathcal{K} \rightarrow \{1,\hdots,m\} ,
	$$
	which lists the elements of $\mathcal{K}$ in the lexicographical order.
	Other sets are possible, \emph{e.g.} downward closed sets which allow for adaptivity, see~\cite{cohen_multivariate_2018}.
	We then define the feature map $\Phi:\R^d\rightarrow\R^m$ such that $\Phi_{i}(\bm x) = \phi_{\sigma^{-1}(i)}(\bm x)$.

	\subsection{Integration of SoS functions}\label{sec:operations_on_SoS_functions}
	Let us rewrite the SoS functions from Eq.~\eqref{eq:SoS_function} to the form $g_A(\bm x)= \trace\left(A \Phi(\bm x)\Phi(\bm x)^\top \rhoref(\bm x)\right)$.
	This allows to apply linear operators $\mathcal L$ acting on functions of $\bm x$, such as integration and differentiation, to $\Phi(\bm x)\Phi(\bm x)^\top\rhoref(\bm x)$ as follow (also see~\cite{marteau-ferey_non-parametric_2020})
	\begin{align*}
		\mathcal L g_A(\bm x)
		= \trace\left(A W (\bm x) \right) ,
		\quad \text{where }W(\bm x) = \mathcal L (\Phi \Phi^\top \rhoref)(\bm x).
	\end{align*}
	As shown in the following proposition, a SoS function remains SoS after integrating over one variable.
	\begin{proposition}[Integration over one variable]\label{propo:SoS_marginalization}
		Let $\mathcal{L}$ be the integration operator over the variable $x_\ell$ with $\ell\in\{1,\hdots,d\}$ defined by
		$$\mathcal{L}g(\bm x) = \int_{\Xc_{\ell}} g(\bm x_{-\ell}, x'_\ell) \d x'_\ell .$$
		Let $g_A$ be a SoS function as in \eqref{eq:SoS_function} with $\rhoref(\bm x)=\prod_{i=1}^d \rhoref_i(x_i)$ and $(\Phi(\bm x))_{\sigma(\alpha)}=\prod_{i=1}^d \phi_{\alpha_{i}}^i(x_i)$, where $\{\phi_{1}^\ell,\phi_{2}^\ell,\hdots\}$ are orthonormal functions in $L^2_{\rhoref_\ell}(\Xc_\ell)$.
		Then we have
		\begin{equation}
			\mathcal{L}g_A(\bm x)
			= \left( \Phi_{-\ell}(\bm x_{-\ell})^\top A_{-\ell} \Phi_{-\ell}(\bm x_{-\ell}) \right) \rhoref(\bm x_{-\ell}),
		\end{equation}
		where $A_{-\ell}$ and $\Phi_{-\ell}(\bm x_{-\ell})$ are respectively the PSD matrix and the feature map defined as follow.
		Denote by $\mathcal{K}_{-\ell} = \{ \alpha_{-\ell} : \alpha \in \mathcal{K}\}$ the multi-index obtained by removing the $\ell$-th components of the elements of $\mathcal{K}$, and by $\sigma_{-\ell}:\mathcal{K}_{-\ell}\rightarrow\{1,\hdots,|\mathcal{K}_{-\ell}|\}$ the lexicographical order for $\mathcal{K}_{-\ell}$.
		The feature map $\Phi_{-\ell}$ is defined by $(\Phi_{-\ell}(\bm x_\ell))_{\sigma_{-\ell}(\alpha_{-\ell})} =  \prod_{i\neq\ell} \phi_{\alpha_{i}}^i(x_i)$.
		The PSD matrix $A_{-\ell}$ is defined by
		\begin{align}
			A_{-\ell}=P \left(A \odot M\right) P^\top,
		\end{align}
		where $\odot$ denotes the Hadamard product and $P\in\R^{|\mathcal{K}_{-\ell}| \times m}$ and $M \in\R^{m \times m}$ are given by
		\begin{align}\label{eq:defP}
			M_{\sigma( \alpha) \sigma( \beta)} = \delta_{\alpha_\ell,\beta_\ell}
			\quad\text{and}\quad
			P_{\sigma_{-\ell}(\alpha_{-\ell})\sigma(\beta)} = \prod_{k\neq \ell}\delta_{\alpha_k , \beta_k} ,
		\end{align}
		for all $\alpha,\beta \in \mathcal{K}$.
	\end{proposition}
	\begin{proof}
		See Appendix~\ref{proof:marginalization_SoS}
	\end{proof}

	The next corollary, given without proof, shows that Proposition \ref{propo:SoS_marginalization} permits to integrate over severall variables.
	
	\begin{corollary}[Integration over severall variables]\label{corrolary:margin_all_d}
		Proposition \ref{propo:SoS_marginalization} can be applied iteratively to integrate over an arbitrary set of variables indexed by $\ell\subset\{1,\hdots d\}$, meaning $\mathcal{L}g_A(\bm x) = \int g_A(\bm x_{-\ell},\bm x_\ell) \d\rhoref_{\ell}$, where $\bm x_\ell = (x_i)_{i\in\ell}$.
		In particular, the integration over all the variables $\mathcal{L}g_A = \int g_A(\bm x) \d\rhoref(\bm x)$ is given by
		\begin{equation}\label{eq:intSoS}
			% 		\int_{\Xc} g_A(\bm x) \d\rhoref(\bm x) =
			\int_{\Xc} \Phi (\bm x)^\top A\Phi (\bm x) \d\rhoref(\bm x) = \trace\left(A\right).
		\end{equation}
	\end{corollary}

	\begin{remark}
		For the fully tensorized set $\mathcal{K} = \mathcal{K}_1 \times\hdots\times \mathcal{K}_d$ with $\mathcal{K}_k = \{1,\hdots, r_k\}$ the matrices $M$ and $P$ as in \eqref{eq:defP} are
		\begin{align*}
			P^\top &= I_{r_1} \otimes \dots \otimes I_{r_{\ell-1}}
			\otimes \bm 1_{r_\ell} \otimes
			I_{r_{\ell+1}} \otimes \dots \otimes I_{r_{d}}, \\
			M &= \bm 1_{r_1\times r_1} \otimes \dots\otimes \bm 1_{r_{\ell-1}\times r_{\ell-1}}
			\otimes I_{r_\ell} \otimes
			\bm 1_{r_{\ell+1}\times r_{\ell+1}} \otimes \dots\otimes \bm 1_{r_\ell\times r_\ell} ,
		\end{align*}
		% 	where $\mathbb{1}$ at the $i$-th position are identity matrices of size $r_i \times r_i$ and
		where $\bm 1_{r_\ell}\in\R^{r_\ell}$ and $\bm 1_{r_\ell}\in\R^{r_\ell\times r_\ell}$ are respectively the vector and the matrix full with ones and $I_{r}$ identity matrices of size $r$.
	\end{remark}
	
	\begin{remark}[Non orthonormal basis]
		If the basis $\{\phi_1^\ell,\phi_2^\ell,\hdots \}$  is not orthonormal, then $M$ from Proposition \ref{propo:SoS_marginalization} is replaced with
		\begin{align*}
			M_{\sigma(\bm i) \sigma(\bm j)} = \int_{\Xc_l} \phi_{i_l}(x_l) \phi_{j_l}(x_l) \d\rhoref_l(x_l).
		\end{align*}
	\end{remark}

	\subsection{SoS densities and Knothe--Rosenblatt map}\label{sec:Rosenblatt}
	
	Consider the normalized SoS function, which we refer to as SoS density, defined by
	\begin{equation}\label{eq:SoSprobaDensity}
		\pi_A(\bm x) = \frac{ \Phi (\bm x)^\top A\Phi (\bm x) }{\trace(A)} \rhoref(\bm x).
	\end{equation}
	By \eqref{eq:intSoS}, $\pi_A$ integrates to one and is a well defined probability density function.
	Next, we show how to compute the Knothe--Rosenblatt (KR) map $\mathcal{Q}_A$ which is is the unique triangular map (up to variable ordering) such that $(\mathcal{Q}_A)_\sharp\rhoref = \pi_A$, see~\cite{rosenblatt_remarks_1952}.
	Let us factorize $\pi_A$ as the product of its conditional marginals
	\begin{align}
		\pi_A(x) = \pi_A(x_1) \pi_A(x_2 | x_1) \dots \pi_A(x_n | x_1, \dots, x_{n-1}),
	\end{align}
	where $\pi_A(x_i | x_1, \dots, x_{i-1})$ is the conditional marginal density given by
	\begin{align}
		\pi_A(x_i | x_1, \dots, x_{i-1}) &= \frac{\pi_A(x_1, \dots, x_{i})}{\pi_A(x_1, \dots, x_{i-1})}, \\
		\pi_A(x_1, \dots, x_{i}) &= \int \pi_A(x_1,\hdots,x_i, x_{i+1}',\hdots,x_d') \d x_{i+1}' \hdots \d x_d' .
	\end{align}
	We also denote by $\Pi_A(x_i | x_1, \dots, x_{i-1})$ the cumulative distribution function (CDF) of the $i$-th conditional of $\pi_A$, given by
	\begin{align}\label{eq:CDF}
		\Pi_A(x_i | x_1, \dots, x_{i-1}) =
		\int_{-\infty}^{x_i} \pi_A(x_i' | x_1, \dots, x_{i-1}) \d x'_i.
	\end{align}
	Let $\overline{\mathcal{Q}}_A: [0,1]^d\rightarrow\Xc$ be the triangular map defined by recursion as follow
	\begin{align}\label{eq:IRT}
		\overline{\mathcal Q}_A(\xi_1,\hdots,\xi_d) =
		\begin{pmatrix}
			x_1\\ x_2\\ \vdots \\ x_d
		\end{pmatrix}
		=
		\left(\begin{array}{l}
			\Pi_A^{-1}( \xi_1  )\\  \Pi_A^{-1}(\xi_2|x_1) \\ \vdots \\ \Pi_A^{-1}(\xi_d|x_1,\hdots,x_{d-1})
		\end{array}\right) ,
	\end{align}
	where $\Pi_A^{-1}( \cdot |x_1,\hdots,x_{i-1})$ is the reciprocal function of the monotone function $\Pi_A( \cdot |x_1,\hdots,x_{i-1})$.
	By construction, the map $\overline{\mathcal Q}_A$ pushes forward the uniform measure $\mu = \mathcal{U}([0,1]^d)$ to $\pi_A$, meaning $(\overline{\mathcal{Q}}_A)_\sharp \mu  = \pi_A $.
	Let $\overline{\mathcal{Q}}_0$ be the KR map of $\rhoref$ such that $(\overline{\mathcal{Q}}_0)_\sharp \mu = \rhoref $. Finally, we let $\mathcal{Q}_A:\Xc\mapsto \Xc$ be the KR map of $\pi_A$ such that
	\begin{align}\label{eq:defQ}
		\left(\mathcal{Q}_A\right)_\sharp \rhoref  = \pi_A ,
		\quad\text{where}\quad
		\mathcal{Q}_A(\bm x) = \overline{\mathcal{Q}}_A \circ\overline{\mathcal{Q}}_0^{-1}(\bm x) .
	\end{align}
	Hence, we can define the set $\mathcal{M}$ of the variational problem in Eq.~\eqref{eq:variational_density_problem} as KR maps parametrized with $A$, given by
	\begin{equation}
		\mathcal{M}_{\text{SoS}} = \{ \mathcal{Q}_A:\Xc\rightarrow\Xc \text{ as in \eqref{eq:defQ}, where } A\succeq0\}.
	\end{equation}
	The procedure of constructing $\overline{\mathcal{Q}}_A$ (and thus $\mathcal{Q}_A$) is summarized in Figure~\ref{fig:diagram_RT_construction}.
	By Proposition \ref{propo:SoS_marginalization}, all the marginals are readily computable and, by using an orthonormal polynomial basis, all the antiderivatives are computable in closed form.
	
	\begin{remark}[$\mathcal{M}_{\text{SoS}}$ contains the identity map]\label{rmk:SoSconstainsID}
		Because $\rhoref$ is a product measure, $\overline{\mathcal{Q}}_0^{-1}(\bm x)$ is a diagonal map with $(\overline{\mathcal{Q}}_{0}^{-1}(\bm x) )_{i} = \int_{-\infty}^{x_i} \rhoref_{i}(x_i')\d x_i'$ being the CDF of $\rhoref_{i}$. In addition, if the basis $\Phi$ contains the constant function, meaning $u_0^\top\Phi(\bm x) = 1$ for some $u_0\in\R^m$, then $\overline{\mathcal{Q}}_0 = \overline{\mathcal{Q}}_{A_0}$ with $A_0=u_0u_0^\top$. Under this condition, the set $\mathcal{M}_{\text{SoS}}$ contains the identify map $\overline{\mathcal{Q}}_{A_0} \circ \overline{\mathcal{Q}}_0^{-1} = id$, hence,
		\begin{align}\label{eq:SoSconstainsID}
			\min_{ \mathcal{Q} \in\mathcal{M}_{\text{SoS}} } \mathrm{D}( \pi  || \mathcal{Q} _\sharp \rhoref)
			\overset{\mathcal{Q} = id}{\leq}
			\mathrm{D}( \pi || \rhoref).
		\end{align}
	\end{remark}

	\begin{figure}%[htpb]
		\includegraphics[width=\textwidth]{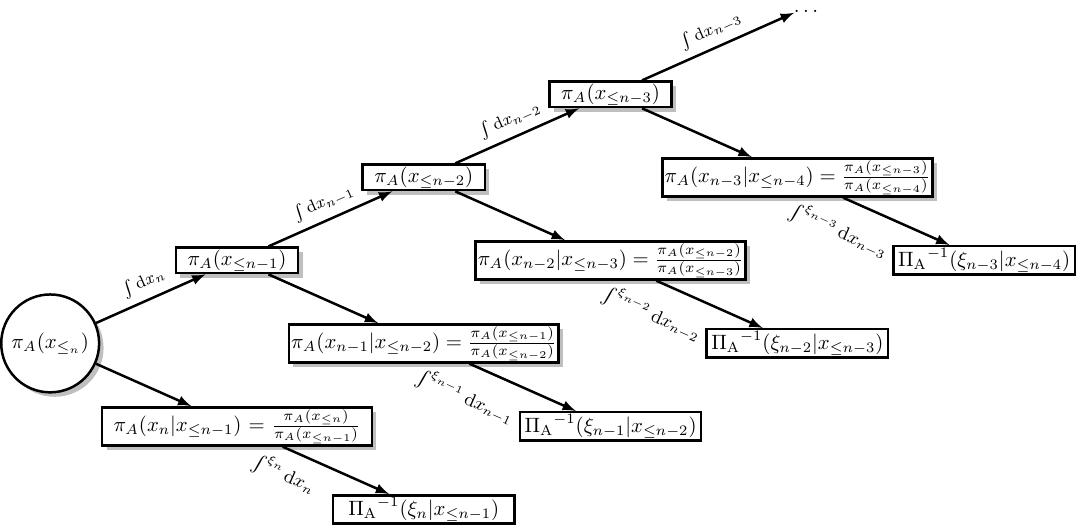}
		\caption{Visualization of the construction of $\overline{\mathcal{Q}}_A$ from a given density $\pi_A$. First, one creates the needed marginalization of $\pi_A$ and the conditionals $\pi_A(x_i|x_{1}, \dots, x_{n-1}) = \pi(x_i | x_{\leq i-1})$. Then, the CDFs are computed by calculating the antiderivatives, which gives access to the RT.
			The inverse RT is constructed by inversion of the CDFs according to formula \eqref{eq:IRT}. 
			%	\red{[OZ: update the figure to accound for \eqref{eq:defQ}; BZ: updated to $\xi$, not sure how to account for \eqref{eq:defQ}]}
		}
		\label{fig:diagram_RT_construction}
	\end{figure}

	\begin{example}[Defining $\rhoref$ on indefinite domains]\label{example:reference_map_indefinite_domain}
		Assume $\mathrm{supp}(\pi) = \R^d$.
		The mappings from Example~\ref{exam:transformed_legendre} can be modified (by an affine change in coordinates from $[-1, 1]$ to $[0, 1]$) to satisfy $(\overline{\mathcal{Q}}_0)_\sharp \mu = \rhoref$.
	\end{example}
	
	\begin{remark}[Mapped basis functions on $\mu$]
		Assume a basis of functions orthonormal in $L^2_{\mu}([0, 1]^d)$ which are mapped to $\Xc$, in the manner of Example~\ref{exam:transformed_legendre}.
		Then we can take $\overline{\mathcal{Q}}_0 = \mathcal{R}$.
		The map $\overline{\mathcal Q}_A$ can be decomposed to $\overline{\mathcal Q}_A = \mathcal{R} \circ \widetilde{\mathcal Q}_A$ where $\widetilde{\mathcal Q}_A : [0, 1]^d \rightarrow [0, 1]^d$. Furthermore,
		\begin{align}
			\mathcal{Q}_A = \mathcal{R} \circ \widetilde{\mathcal Q}_A \circ \mathcal{R}^{-1}.
		\end{align}
		Composing such maps can be done efficiently using $\mathcal{Q}_{A_1}\circ\hdots\circ \mathcal{Q}_{A_L} = \mathcal{R} \circ (\widetilde{\mathcal{Q}}_{A_1}\circ\hdots\circ\widetilde{\mathcal{Q}}_{A_L})\circ  \mathcal{R}^{-1}$.
	\end{remark}
	
	\subsection{Conditional SoS maps}\label{subsec:conditional_densities}
	The triangular structure allows for conditional density estimation, for example, see~\cite{cui_scalable_2023, marzouk_introduction_2016, baptista2023representation, hosseini2023conditional}.
	Consider a target joint density $\pi(\bm x, \bm y)$ from which we want to sample and evaluate the conditional density $\pi(\bm x|\bm y)$.
	Assuming we have an approximation $\widetilde\pi(\bm x, \bm y)$ of $\pi(\bm x, \bm y)$ given by $\widetilde\pi(\bm x, \bm y) = \mathcal{Q}_\sharp \rhoref (\bm x, \bm y)$ where $\mathcal{Q}$ is a block triangular map as in
	\begin{align}\label{eq:triangular}
		\mathcal{Q}(\bm \xi_x, \bm \xi_y) = \left(\begin{array}{l}
				Q_{1}(\bm \xi_y) \\
				Q_{2}(\bm \xi_y, \bm \xi_x)
			\end{array}\right) ,
	\end{align}
	and where $\rhoref(\bm \xi_x,\bm \xi_y) = \rho_x(\bm \xi_x) \rho_y(\bm \xi_y)$ is a product reference distribution.
	Then, the so-called \emph{stochastic map filter} \cite{spantini2022coupling}, defined by
	\begin{align}
		Q_{x|y}(\bm \xi_x|\bm y) = Q_{2}(Q_1^{-1}(\bm y), \bm \xi_x) ,
	\end{align}
	is such that $Q_{x|y}(\cdot|\bm y)_\sharp\rho_x(\bm x) = \widetilde\pi(\bm x | \bm y)$.
	Notice that the triangular structure \eqref{eq:triangular} is preserved by composition of maps with the same triangular structure.
	It is also worth to note that when using the KL divergence ($\alpha=1$), the relation $\E\left[\KL\left( \pi(\cdot|\bm Y) || \widetilde\pi(\cdot|\bm Y) \right) \right] \leq \KL\left( \pi || \widetilde\pi \right)$ holds for $\bm Y\sim\pi(\bm y)$. This means that controlling the KL divergence between the joint densities permits to control the KL divergence between the conditionals in expectation over $\bm Y$.

	\section{Sequential transport maps using $\alpha$-divergence}\label{sec:SequentialTM}

	We now build transport maps sequentially as in \eqref{eq:variational_density_problem} by using an arbitrary sequence of bridging densities $\{\pibridge{\ell}\}_{\ell=1}^L$.
	When the target density $\pi(\bm x)\propto \mathcal{L}(\bm x) \pi_0(\bm x)$ as in \eqref{eq:deff} can be evaluated up to a normalizing constant, a natural choice of bridging densities are \emph{tempered bridging densities},
	\begin{equation}\label{eq:tempered}
	 \pibridge{\ell}(\bm x) =  \mathcal{L}(\bm x)^{\beta_\ell} \pi_0(\bm x) ,
	\end{equation}
	with $0 < \beta_1 \leq \dots \leq \beta_L = 1$ arbitrarily fixed so that $\pi^{(L)}\propto\pi$.
	Other choices are also possible depending on the application, see for instance \cite{del_moral_sequential_2006} for filtering problems or~\cite{cui2024deep} for rare event estimation problems.
	When only samples $\bm X^{(1)},\hdots,\bm X^{(N)}$ from $\pi$ are given, a natural choice of bridging densities are the \emph{diffusion-based bridging densities} of the form
	\begin{equation}\label{eq:density_diffusion}
		\pibridge{\ell}(\bm x) = \int \kappa_{t_\ell}(\bm x,\bm y) \pitar(\bm y)\d \bm y ,
	\end{equation}
	where $\kappa_t(\bm x,\bm y) \propto \exp(-\frac{\|\bm x-e^{-t} \bm y\|^2}{2(1-e^{-2t})})$ is the kernel associated with the Ornstein-Uhlenbeck diffusion process and $t_1\geq \hdots \geq t_L=0$ is an arbitrary sequence so that $\pibridge{L}=\pi$.
	This corresponds to the diffusion models~\cite{sohl2015deep,song2021scorebased}. Samples from $\pibridge{\ell}$ as in \eqref{eq:density_diffusion} are thus trivially generated as follow
	\begin{equation}\label{eq:blurred_samples}
	 \bm X^{(i)}_\ell = \exp(-t_\ell)  \bm X^{(i)} + \sqrt{1-\exp(-2 t_\ell)}  \bm Z^{(i)} ,
	\end{equation}
	where $\bm Z^{(i)}\sim\mathcal{N}(0,I_d)$ are independent samples drawn from the standard Gaussian distribution.
	Note that diffusion-based bridging densities \eqref{eq:density_diffusion} are also considered in \cite{grenioux2024stochastic,akhound2024iterated} in order to sample from a given density.
	Figure~\ref{fig:vis_ornstein_uhlenbeck} gives a visual comparison of the tempered~\eqref{eq:tempered} and diffusion-based~\eqref{eq:density_diffusion} bridging densities for a two--mode Gaussian target density $\pi$.
	\begin{figure}
		\includegraphics[width=\textwidth]{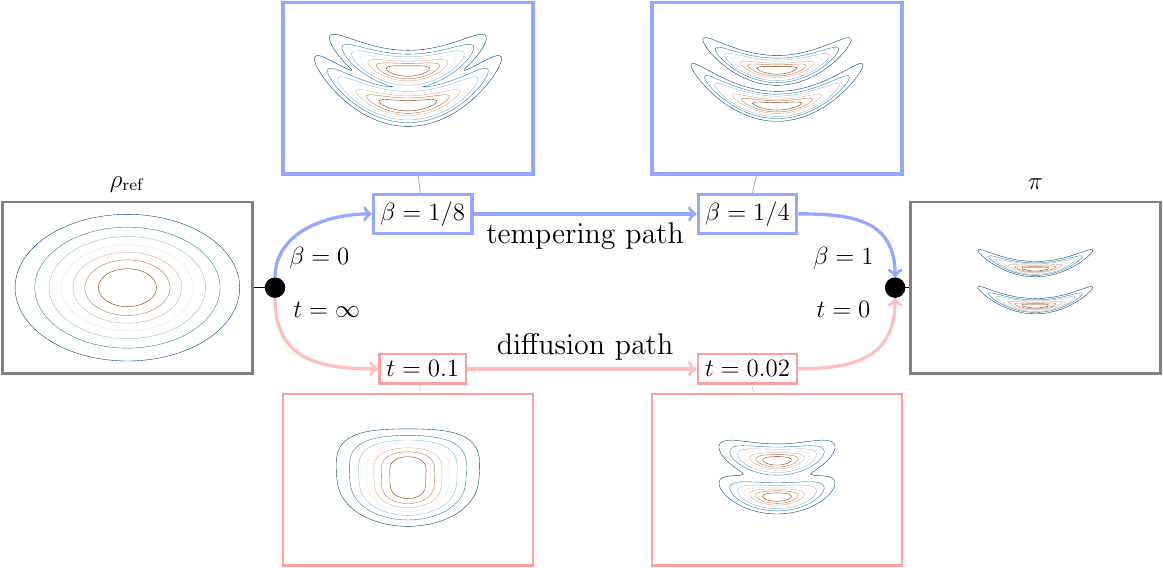}
		\caption{Visual comparison of tempered (top) and diffusion based (bottom) bridging densities for a banana distribution (right) with a Gaussian reference distribution (left).}
		\label{fig:vis_ornstein_uhlenbeck}
	\end{figure}

	Let us assume that the $\alpha$-divergence between two consecutive bridging densities is bounded by a quantity which decreases with the number $L$ of bridging densities.

	\begin{assumption}\label{assu:bound_divergence_of_bridging_densities}
		The sequence of bridging densities $\{\pibridge{\ell}\}_{\ell=1}^L$ satisfies
		\begin{align}\label{eq:bound_divergence_of_bridging_densities}
			\AlphaDiv\left(\pibridge{\ell} || \pibridge{\ell-1}\right) \leq \eta(L) ,
		\end{align}
		for all $1\leq\ell\leq L$, where $\eta(L)\geq0$ is such that $\eta(L)\rightarrow0$ when $L\rightarrow \infty$.
		Here we used the convention that $\pibridge{0}=\rhoref$.
	\end{assumption}

	Intuitively, Assumption \ref{assu:bound_divergence_of_bridging_densities} means that the complexity of learning $\pi$ is divided into $L$ sub-problems of much lower complexity compared to learning $\pi$ directly.
	Based on this idea, the bridiging densities $\{\pibridge{\ell}\}_{\ell=1}^L$ must be chosen in a way that every sub-problems are ideally such that $\AlphaDiv\left(\pibridge{\ell} || \pibridge{\ell-1}\right) = \eta(L)$ for all $\ell\leq L$. In principle, this can be obtained with an appropriate choice of tempering parameter $\beta_\ell$ in \eqref{eq:tempered} or with an appropriate choice of time step $t_\ell$ in \eqref{eq:density_diffusion}.
	We discuss this further in Section \ref{sec:scheduler}.

	Now we introduce a sequence of diffeomorphic map $\mathcal{Q}_1,\hdots,\mathcal{Q}_L$ so that
	\begin{equation}\label{eq:pitilde}
	 \pibridgetilde{\ell} = (\mathcal{T}_{\ell})_\sharp \rhoref
	 \quad\text{where}\quad
		\mathcal{T}_{\ell} = \mathcal{Q}_{1}\circ\hdots\circ \mathcal{Q}_{\ell}.
	\end{equation}
	\begin{assumption}\label{assump:better_by_w}
		We assume there exists $\omega < 1$ such that
		\begin{align}\label{eq:tmp135908}
			\AlphaDiv\left( \pibridge{\ell}|| \pibridgetilde{\ell} \right) &\leq \omega \AlphaDiv\left( \pibridge{\ell} || \pibridgetilde{\ell-1} \right),
		\end{align}
		for all $1\leq\ell\leq L$ and $\AlphaDiv(\pibridge{0} || \rhoref)\leq \omega$, where we use the convention $\pibridgetilde{0} = \rhoref$.
	\end{assumption}

	By definition \eqref{eq:pitilde} and by Property \eqref{eq:TransportProperty} of $\alpha$-divergences, this is equivalent to
	\begin{align}\label{eq:better_by_w}
% 			\AlphaDiv\left( \pibridge{\ell}|| \pibridgetilde{\ell} \right) &\leq \omega \AlphaDiv\left( \pibridge{\ell} || \pibridgetilde{\ell-1} \right),
			\AlphaDiv\left(\left. \overline{\pi}^{(\ell)} \right\| (\mathcal{Q}_{\ell})_\sharp\rhoref \right) &\leq \omega \AlphaDiv\left(\left. \overline{\pi}^{(\ell)} \right\| \rhoref \right),
		\end{align}
	where $\overline{\pi}^{(\ell)}=(\mathcal{T}_{\ell-1})^\sharp\pibridge{\ell}$ is the pullback density of the $\ell$-th bridiging density by the $(\ell-1)$-th map $\mathcal{T}_{\ell-1}=\mathcal{Q}_{1}\circ\hdots\circ \mathcal{Q}_{\ell-1}$.
	In the light of Remark \ref{rmk:SoSconstainsID}, Assumption
	\ref{assump:better_by_w} is stronger than \eqref{eq:SoSconstainsID}.
	It is worth to note that, for the Hellinger distance ($\alpha=1/2$), Theorem 2.8 from \cite{westermann_measure_2023} gives a convergence rate $\omega=\omega(n)$ for the approximation of smooth densities by a polynomials squared constructed on an adaptive polynomial basis with $n$ term.
	Because the SoS function \eqref{eq:SoS_function} contains the set of polynomial square (simply by letting $A$ be a rank-1 matrix), the results from \cite{westermann_measure_2023} transfers to SoS.
	For instance, using Theorem 2.8 from \cite{westermann_measure_2023}, we have that if $\bm x\mapsto (\overline{\pi}^{(\ell)}(x_1,\hdots,x_d))^{1/2}$ is analytic, then $\omega(n)=\mathcal{O}(\exp(-\beta n^{1/d}))$ for some $\beta>0$.

	Next we assume a triangular like property for the $\alpha$-divergence.
	\begin{assumption}\label{assump:symetric_proj_diff}
		We assume that there exists a constant $\epsilon \geq 0$ so that
		\begin{align}
			\AlphaDiv\left(\pibridge{\ell} || \pibridgetilde{\ell-1}\right)
			\leq  (1+\epsilon) \AlphaDiv\left(\pibridge{\ell} || \pibridge{\ell-1}\right) + \AlphaDiv\left(\pibridge{\ell-1} || \pibridgetilde{\ell-1}\right),
			\label{eq:tmp01375861}
		\end{align}
		for all $1\leq\ell\leq L$.
	\end{assumption}
	In general, such triangle inequality for $\alpha$-divergence does not hold.
	We discuss this assumption further in Section \ref{subsec:convergence_analysis_geometry}.
	We now establish a convergence bound for $\alpha$-divergences.
	\begin{theorem}\label{th:chiu_bound}
		Under Assumptions~\ref{assu:bound_divergence_of_bridging_densities},~\ref{assump:better_by_w} and~\ref{assump:symetric_proj_diff}, we have
		\begin{align}\label{eq:chiu_bound}
			\AlphaDiv\left( \pibridge{L} || \pibridgetilde{L}\right) \leq \frac{\omega (1 + \epsilon)}{1 - \omega } \eta(L).
		\end{align}
	\end{theorem}
	\begin{proof}
		We can write
		\begin{align*}
		\AlphaDiv (\pibridge{\ell} || \pibridgetilde{\ell})
		&\overset{\eqref{eq:tmp135908}}{\leq} \omega \AlphaDiv (\pibridge{\ell} || \pibridgetilde{\ell-1}) \\
		&\overset{\eqref{eq:tmp01375861}}{\leq} \omega \left(  (1+\epsilon) \AlphaDiv(\pibridge{\ell} || \pibridge{\ell-1}) + \AlphaDiv(\pibridge{\ell-1} || \pibridgetilde{\ell-1}) \right)\\
		&\overset{\eqref{eq:bound_divergence_of_bridging_densities}}{\leq} \omega (1+\epsilon) \eta(L) + \omega \AlphaDiv(\pibridge{\ell-1} || \pibridgetilde{\ell-1}) .
		\end{align*}
		A direct recursion gives
		\begin{align*}
		 \AlphaDiv\left( \pibridge{L} || \pibridgetilde{L}\right)
		 &\leq  (1+\epsilon) \eta(L) \left(\omega + \omega^2 + \hdots + \omega^{L-1} \right)
		\end{align*}
		where we used the fact that $\AlphaDiv(\pibridge{1} || \pibridgetilde{1})\leq \omega \eta(L)$ by Assumptions \ref{assu:bound_divergence_of_bridging_densities} and \ref{assump:better_by_w}.
		This gives \eqref{eq:chiu_bound} and concludes the proof.
	\end{proof}

	\begin{remark}[Hellinger distance]\label{remark:comparison_error_bound_Hellinger_geometric}
		The case $\alpha=1/2$, which corresponds to the Hellinger distance, can be treated differently without using Assumption \ref{assump:symetric_proj_diff}. Indeed, since $\mathrm{D}_{1/2}(f||g) = 2\int ( \sqrt{f}-\sqrt{g})^2\d \bm x$, we have that  	$(f,g)\mapsto\sqrt{\mathrm{D}_{1/2}(f||g)}$ satisfies a triangle inequality
		$
		 \sqrt{\mathrm{D}_{1/2}(f||g)} \leq \sqrt{\mathrm{D}_{1/2}(f||h)} + \sqrt{\mathrm{D}_{1/2}(h||g)} .
		$
		Using this inequality instead of \eqref{eq:tmp01375861} in the proof of Theorem \ref{th:chiu_bound}, we obtain
		\begin{equation}
		 \mathrm{D}_{1/2}(\pibridge{L}||\pibridgetilde{L}) = \frac{\omega}{(1-\sqrt{\omega})^2}\eta(L),
		\end{equation}
		see also the analysis proposed in \cite{cui_scalable_2023,cui_self-reinforced_2023} for the Hellinger distance $\Hell( \pi_f || \pi_g ) = (\frac{1}{2}\int ( \sqrt{\pi_f}-\sqrt{\pi_g} )^2\d \bm x )^{1/2}$.
		Compared to \eqref{eq:chiu_bound}, the denominator is larger, but there is no extra-factor $(1+\epsilon)$ in the numerator.
	\end{remark}

\subsection{Validation of Assumption \ref{assu:bound_divergence_of_bridging_densities}}\label{sec:scheduler}
We show existence of tempered and diffusion bridging densities that satisfy Assumption~\ref{assu:bound_divergence_of_bridging_densities} with
\begin{equation}
	\eta(L)=\mathcal{O}(1/L^2)
\end{equation}
for $L \rightarrow \infty$.
Based on our results, we also provide heuristics for choosing sequences of $\beta_\ell$ and $t_\ell$.
We refer to this choice as \emph{scheduling}, which is already a coined term in the diffusion model literature (e.g. see~\cite{kingma2021variational, nichol_improved_2021}).

We start with tempered bridging densities and the following proposition, whose proof is given in Appendix~\ref{proof:scheduler_tempering}.

	\begin{proposition}\label{prop:scheduler_tempering}
	Let $u\mapsto\beta(u)$ be a function such that $\beta(0)=0$ and $\beta(1)=1$ and
	\begin{equation}\label{eq:scheduler_tempering_unnormalized}
		\beta'(u) = \Omega C''(\beta(u))^{-1/2} ,
	\end{equation}
	for some constant $\Omega\geq0$, where $C(\beta)= \int \d\pi_\beta$ is the normalizing constant of the density $\pi_\beta(\bm x)=\mathcal{L}(\bm x)^\beta \pi_0(\bm x)$ and $C''(\beta)=\int \log(\mathcal{L})^2 \d\pi_\beta$ its second-order derivative.
	Let $\beta_0\leq \beta_1\leq \hdots\leq \beta_L$ be a sequence defined by
	\begin{equation}\label{eq:scheduler_beta}
		\beta_\ell = \beta( \ell/L ) ,
	\end{equation}
	Then the tempered bridiging density $\pibridge{\ell}(\bm x)=\mathcal{L}(\bm x)^{\beta_\ell} \d\pi_0(\bm x)$ satisfies
	\begin{equation}\label{eq:control_tempering_unnormalized}
	  \AlphaDiv\left(\pibridge{\ell} || \pibridge{\ell-1}\right) = \frac{\Omega^2}{2L^2} + \mathcal{O}\left(\frac{1}{L^{3}}\right) .
	\end{equation}
	\end{proposition}
	The scheduler \eqref{eq:scheduler_tempering_unnormalized} requires knowing the normalizing constant $C(\beta)$ of $\pi_\beta$ (and its second order derivative), which we may not have access to in practice. We propose the following heuristics
	\begin{itemize}
	\item Since $C(\beta)= \int \mathcal{L}(x)^\beta \d\pi_0(x)$, a naive estimate is $C(\beta) \approx C(1)^\beta$. This estimate is exact at $\beta=1$ and $\beta=0$, provided $\pi_0$ is a normalized probability density. We deduce $C''(\beta) \approx C(1)^\beta \log(C(1))^2$, so that condition \eqref{eq:scheduler_tempering_unnormalized} becomes $\beta'(u) \propto C(1)^{-\beta/2}$ and thus $\beta(u)$ is necessary of the form
	\begin{align}
		\beta(u) = \frac{2\log(1+ (C(1)^{1/2}-1) u)}{\log(C(1))}
	\end{align}
	\item Another option is to assume $C''(\beta)\propto 1/\beta^2$. That way \eqref{eq:scheduler_tempering_unnormalized} becomes $\beta'(u)\propto\beta(u)$ and then
	\begin{align}
		\beta(u) \propto \exp( a u)
	\end{align}
	for some constant $a>1$. Then \eqref{eq:scheduler_beta} yields $\beta_{\ell+1} = \rho_L\beta_\ell$ with $\rho_L=\exp(a/L)$, which is the exponential scheduler used in \cite{cui_deep_2021,cui_scalable_2023}.

	\end{itemize}

	Next, we discuss the case for diffusion bridging densities.

	\begin{proposition}\label{prop:scheduler_diffusion}
		Given a target density $\pi$ and $t\geq0$, consider $\pi_t(\bm x) = \int \kappa_t(\bm x,\bm y)\pi(\bm y) \d \bm y$  where $\kappa_t(\bm x,\bm y)\propto\exp\left( -\frac{\|\bm x-e^{-t}\bm y\|^2}{2(1-e^{-2t})}\right)$.
		Let $u\mapsto t(u)$ be a function such that $t(0)=\infty$ and $t(1)=0$ and which satisfies $t'(u) = - \Omega D(t(u))^{-1}$ for some $\Omega\geq0$, where
		\begin{equation}\label{eq:D}
		 D(t) =
		 \left( \int \left(\nabla\log\pi_t^\top \left(\nabla\log\frac{\pi_t}{\pi_\infty}\right)  +  \trace\left(\nabla^2\log \frac{\pi_t}{\pi_\infty} \right) \right)^2 \d \pi_{t} \right)^{1/2} ,
		\end{equation}
		where $\pi_\infty = \mathcal{N}(0,I_d)$.
		Then, the diffusion bridiging densities $\pi^{(\ell)}=\pi_{t_\ell}$ with $t_\ell$ defined as
		\begin{equation}
		 t_\ell = t(\ell/L),
		\end{equation}
		satisfy
		\begin{equation}\label{eq:scheduler_diffusion}
		D_\alpha(\pi^{(\ell+1)}||\pi^{(\ell)}) = \frac{\Omega^2}{2 L^2}
		+ \mathcal{O}\left(\frac{1}{L^3}\right) .
		\end{equation}

	\end{proposition}
	The proof is given in Appendix~\ref{proof:scheduler_diffusion}.
	The quantity $D(t)$ is a measure of how close $\pi_t$ is to $\pi_\infty = \mathcal{N}(0,1)$. This convergence of $\pi_t$ to $\pi_\infty$ is known to be exponentially fast, see \emph{e.g.}~\cite{bakry2014analysis}. As to reflect this, we propose the following heuristic
	\begin{equation}
	D(t) \approx \frac{1}{ \exp(\rho t) - B} ,
	\end{equation}
	for some constants $\rho>0$ and $0\leq B < 1$. That way, the condition $t'(u) \propto - 1/ D(t(u))$ simplifies to
	$t'(u) \propto - ( \exp(\rho t(u)) - B)$ and yields the following scheduler
	\begin{equation}\label{eq:proposed_time_diffusion}
		t(u) = \frac{1}{\rho} \log\left( \frac{ B}{ 1-(1-B)^u }\right) .
	\end{equation}

	\subsection{Geometric interpretation of Assumption \ref{assump:symetric_proj_diff}}\label{subsec:convergence_analysis_geometry}

	We explain Assumption~\ref{assump:symetric_proj_diff} using the information geometric properties of $\alpha$-divergences. We refer to~\cite{amari_information_2016} for detailed introduction on information geometry.
	$\alpha$-divergences posses geodesics contained in the space of measures, called $\alpha$-geodesics. Let us state the definition of these geodesics.
	\begin{definition}[$\alpha$ and $\alpha^*$-geodesic~\cite{amari_information_2016}]
		The $\alpha$-geodesic connecting two (possibly unnormalized) densities $f$ and $g$ is given by
		\begin{align*}
			\mu_{t}(\bm x) = \begin{cases}
				\left\{(1-t)f(\bm x)^{1-\alpha} + t g(\bm x)^{1-\alpha} \right\}^{\frac{1}{1-\alpha}} \qquad &\text{for } \alpha \neq 1 \\
				f(\bm x)^{1-t} g(\bm x)^t &\text{for } \alpha = 1
			\end{cases}
		\end{align*}
		for $0\leq t\leq 1$.
		Furthermore, the $(1-\alpha)$-geodesic is the $\alpha^*$-geodesic (dual $\alpha$-geodesic).
	\end{definition}
	In our analysis, we make use of the generalized Pythagorean theorem and projection theorem, which we state in the following.
	\begin{theorem}[generalized Pythagorean theorem~\cite{amari_information_2016}]\label{th:generalizedPythagoreantheorem}
		Given three densities $f,g,h$  so that the $\alpha$-geodesic between $f$ and $g$ and the $\alpha^*$-geodesic between $g$ and $h$ are orthogonal. Then,
		\begin{align}
			\AlphaDiv\left(f || h \right) = \AlphaDiv\left(f || g\right) + \AlphaDiv\left(g || h\right).
		\end{align}
	\end{theorem}
	\begin{theorem}[Projection theorem~\cite{amari_information_2016}]
		Consider a density $f$ and a set of density $\mathcal{M}$.
		If $\mathcal{M}$ is an $\alpha^*$-flat, meaning that $\mathcal{M}$ contains all the $\alpha^*$-geodesic connecting any two points of $\mathcal{M}$, then
		\begin{align}
			\min_{ g \in \mathcal{M}} \AlphaDiv\left(f || g\right)
		\end{align}
		admits a unique solution which is the $\alpha$-projection of $f$ onto $\mathcal{M}$.
	\end{theorem}
	Let us define an $\alpha$-geodesic $\mu_{t}$ which connects two consecutive bridging densities $\pibridge{\ell}$ and $\pibridge{\ell+1}$, as shown in Figure~\ref{fig:vis_generalized_pythagorean}.
	An approximation of $\pibridge{\ell}$ by $\pibridgetilde{\ell}$ is found by minimizing the variational density estimation problem, which in the geometric interpretation is a projection from $\pibridge{\ell}$ onto the approximation manifold $\mathcal{M}$, following an $\alpha$-geodesic.
	From an information geometric view, this has an unique solution, if $\mathcal{M}$ is $\alpha^*$-flat.
	This is \emph{e.g.} the case for the set $\mathcal{M}$ of SoS densities  \eqref{eq:SoS_function} when $\alpha=1$, but also for the set of squared function $\mathcal{M}=\{g(\bm x) = p(\bm x)^2\rhoref(\bm x), p\in\text{vector space} \}$  for $\alpha=1/2$ which are used in~\cite{cui_scalable_2023, westermann_measure_2023, cui_self-reinforced_2023}.
	
	Next, we project the approximation $\pibridgetilde{\ell}$ back onto a point $f_{\mathrm{proj}}^{(\ell)}$ on the $\alpha$-geodesic between $\pibridge{\ell}$ and $\pibridge{\ell+1}$. 
	This projection is defined by
	\begin{align}\label{eq:backwards_projection_measure}
		f_{\mathrm{proj}}^{(\ell)} = \argmin_{\mu_t \in \alpha\mathrm{-geodesic}(\pi^{(\ell)},\pi^{(\ell+1)})} \AlphaDiv\left(\mu_t || \pibridgetilde{\ell}\right).
	\end{align}
	The density $f_{\mathrm{proj}}^{(\ell)}$ is unique since it is defined by an $\alpha^*$-projection on a $\alpha$-geodesic.
	This allows us to use the generalized Pythagorean theorem between $\pibridgetilde{\ell}$, $f_{\mathrm{proj}}^{(\ell)}$, and $\pibridge{\ell+1}$.
	Based on these ideas, the next proposition allows us to explain Assumption~\ref{assump:symetric_proj_diff} from an information geometric point of view.

	\begin{figure}
		\centering
		\includegraphics{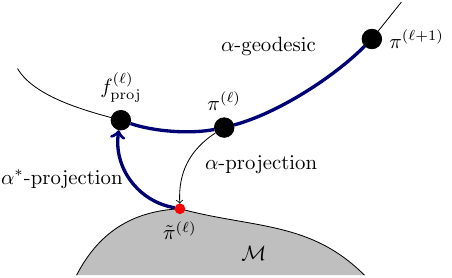}
		\caption{Visualization of the $\alpha$-geodesic going through $\pibridge{\ell}$ and $\pibridge{\ell+1}$ as well as the approximation submanifold $\mathcal{M}$ with the approximation $\pibridgetilde{\ell}$ being the $\alpha$-projection of $\pibridge{\ell}$ onto $\mathcal{M}$. The $\alpha^*$-projection back on the $\alpha$-geodesic is used in order to use the generalized Pythagorean theorem between $\pibridgetilde{\ell}$, $f_{\text{proj}}^{(\ell)}$, and $\pibridge{\ell+1}$.}\label{fig:vis_generalized_pythagorean}
	\end{figure}
	\begin{proposition}\label{prop:symetric_proj_diff}
		We assume that there exists an $\epsilon \geq 0$ so that
		\begin{align}
			\AlphaDiv\left(\pibridge{\ell+1} || f_{\mathrm{proj}}^{(\ell)}\right) \leq \left(1 + \epsilon\right) \AlphaDiv\left(\pibridge{\ell+1} || \pibridge{\ell} \right), \label{eq:tmp30198571}
		\end{align}
		for all $1\leq \ell\leq L$.
		Then Assumption \ref{assump:symetric_proj_diff} is satifed with the same $\epsilon$.
	\end{proposition}
	\begin{proof}
		By construction, the generalized Pythagorean theorem \ref{th:generalizedPythagoreantheorem} gives
\begin{align*}
 D_\alpha(\pi^{(\ell+1)}||\widetilde \pi^{(\ell)})
 &= D_\alpha(\pi^{(\ell+1)}|| f_{\mathrm{proj}}^{(\ell)} ) + D_\alpha( f_{\mathrm{proj}}^{(\ell)} ||\widetilde \pi^{(\ell)}) \\
 &\overset{\eqref{eq:backwards_projection_measure}}{\leq}  D_\alpha(\pi^{(\ell+1)}|| f_{\mathrm{proj}}^{(\ell)} ) + D_\alpha( \pi^{(\ell)} ||\widetilde \pi^{(\ell)}) \\
 &\overset{\eqref{eq:tmp30198571}}{\leq}  (1+\varepsilon) D_\alpha(\pi^{(\ell+1)}|| \pi^{(\ell)} ) + D_\alpha( \pi^{(\ell)} ||\widetilde \pi^{(\ell)})
\end{align*}
which gives \eqref{eq:tmp01375861} and concludes the proof.
	\end{proof}
	
	\section{Implementing sequential maps for high-dimensional and data-driven problems }\label{sec:highdim}
	In this section, we bridge the gap between the previous presented methodology and analysis to its application.
	First, in Section~\ref{subsec:towards-high-dim}, we present a method to extend the presented method to higher dimensions, then, in Section~\ref{subsec:from_data} we develop an algorithm for learning sequential transport maps from data.

	\subsection{Towards high dimensions using subspace projections: the lazy maps}\label{subsec:towards-high-dim}
	The SoS densities presented do not scale well in high dimension due to the exponential increase of $|\mathcal{K}|$ with respect to $d$, see \eqref{eq:setK}.
	One strategy is to iteratively select a subspace of $\Xc$ in which to perform density estimation.
	In the following, we assume $\Xc = \R^d$.
	Given a matrix $U_\ell \in \R^{d \times r}$ with $r\ll d$ orthogonal columns, we consider a map of the form
	\begin{align}
		\mathcal{Q}_\ell(\bm x) = U_\ell \widetilde{\mathcal{Q}}_\ell(U_\ell^\top \bm x) + \left(1 - U_\ell U_\ell^\top\right) \bm x ,
	\end{align}
	where $\widetilde{\mathcal{Q}}_\ell:\R^r\rightarrow\R^r$ is a low-dimensional map acting on $\R^r$.
	By construction, $\mathcal{Q}_\ell$ is active on the low-dimensional subspace $\text{range}(U_\ell)$, hence called a \emph{lazy map} \cite{brennan_greedy_2020}. By iteratively working on different subspaces, density estimation can be performed in high dimensions.
	
	When working with unnormalized densities,~\cite{brennan_greedy_2020} identifies $U_\ell$ using the score function of the target density. When working with dataset, a similar approach was used in~\cite{baptista_dimension_2022}, where score function is first estimated using a neural network.
	In our work, we chose $U_\ell$ uniformly random on the Stiefel manifold. We let an exploration of other methods open for future work.
	
	\begin{remark}[Lazy maps for conditional density estimation]
		It is also possible to use lazy maps for conditional density estimation.
		In this case, two projections are needed, one on the conditioned variables and a second one on the others.
		Hence, a conditional map writes
		\begin{align}
			\mathcal{Q}_\ell\left(\begin{array}{c}
				\bm x \\
				\bm y
			\end{array}\right) = U_{\ell} \widetilde{\mathcal{Q}}_\ell \left( U_{\ell}^\top \left(\begin{array}{c}
				\bm x \\
				\bm y
			\end{array}\right)\right) + \left(1 - U_{\ell} U_{\ell}^\top \right) \left(\begin{array}{c}
				\bm x \\
				\bm y
			\end{array}\right),
		\end{align}
		where $U_{\ell}$ is a block-matrix of the form
		\begin{align}
			U_{\ell} =
			\begin{pmatrix}
				U_{\ell}^{\Xc} & 0 \\
				0 & U_{\ell}^{\Yc}
			\end{pmatrix},
			% 			\end{array}\right)
	\end{align}
	with $U_l^{\Xc} \in \R^{d_x \times r_x}$ is the reduced dimension in $\bm x$ and $U_l^{\Yc} \in \R^{d_y \times r_y}$ is the reduced dimension in $\bm y$.
	\end{remark}

	\subsection{Sequential transport maps from data using diffusion}\label{subsec:from_data}
	
	Learning an inverse diffusion from data would learn a distribution which, when $t\rightarrow0$, is concentrated on a \emph{finite amount} of data points.
	As to avoid this, \cite{nichol_improved_2021} proposes to stop the inverse diffusion at a very small time.
	Similarly, we propose to use smaller and smaller time steps when $t$ approaches $0$, and use a stopping criterion based on cross validation.
	To do so, we modify the bridging density time function $t(u)$ from Eq.~\eqref{eq:proposed_time_diffusion} to
	\begin{equation}
		t_\text{data}(u) = \frac{1}{\rho} \log\left( \frac{1}{ 1-(1-B)^u }\right) ,
	\end{equation}
	so that $t_\text{data}(u) \rightarrow 0$ for $u \rightarrow \infty$ instead of $u \rightarrow 1$.
	\begin{figure}
		% 	[htbp]
		\centering
		\includegraphics[width=0.6\textwidth]{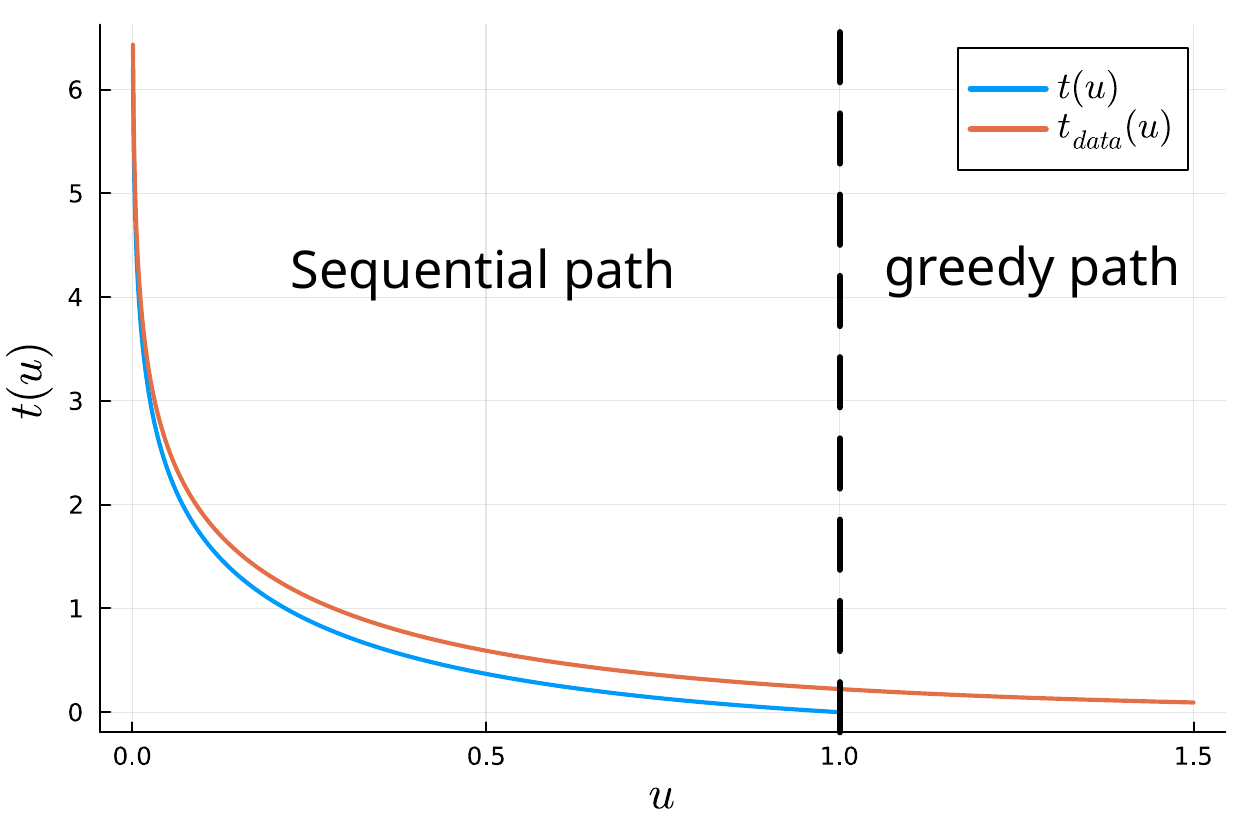}
		\caption{Comparison of $t(u)$ and $t_\text{data}(u)$ for $B=0.8$ and $\rho=1$.}\label{fig:diffusion_time_data}
	\end{figure}
	This scheduler is similar to $t(u)$ when $u \ll 1$, while for larger $u\gg 1$, the bridging density merely change, see Figure~\ref{fig:diffusion_time_data}.
	Given a $L_0$, we define the bridging densities by
	\begin{align}
		\pi^{(\ell)} = \pi_{t_\ell}, \qquad \text{with}\quad t_l = t_\text{data}(\ell/L_0) \quad \text{for}\quad \ell \geq 1.
	\end{align}
	With this choice of bridging densities, $\ell\leq L_0$ recovers the \emph{sequential regime} as analyzed by Theorem \ref{th:chiu_bound}, wheras $\ell\geq L_0$ corresponds to a \emph{greedy regime} where, by Assumption \ref{assump:better_by_w} with $\pi^{(\ell)}\approx\pi$, we obtain the convergence
	$\AlphaDiv ( \pi|| \pibridgetilde{\ell})=\mathcal{O}(\omega^\ell)$.
	
	For the cross validation, we split the data $X = \left\{x_i\right\}_{i=1}^N$, with $x_i \sim \pitar$, into a training $X_{\text{train}}$ and validation set $X_{\text{val}}$ and evaluate the negative log-likelihood on the validation data.
	\begin{algorithm}[h]
		\caption{Sequential transport from data with diffusion process}\label{algo:learning_from_data}
		\begin{algorithmic}[1]
			\Require{$X_{\text{train}}$, $X_{\text{val}}$, $L_0$}
			\State Initialize $\ell = 0$ and $\mathcal{T}_0 = \mathrm{id}$
			\State Define the negative log-likelihood $R(\mathcal{T}) = - \sum_{x\in X_{\text{val}}} \ln \mathcal{T}_\sharp \rhoref(x)$
			\While{$\ell\leq L_0$ or $R(\mathcal{T}_{\ell}) \leq R(\mathcal{T}_{\ell-1})$}
			\State $\ell \gets \ell + 1$
			\State $t_\ell = t_\text{data}(\ell/L_0)$\;
			\State Evolve $X_\ell = \{x^{(\ell)}_i\}$ according to Eq.~\eqref{eq:blurred_samples}
			\State Compute $\pibridgetilde{\ell}(x) = \argmin_{\tilde{\pi} \in \mathrm{SoS}} - \sum_{x \in X_{\text{train}}} \log \pibridgetilde{\ell}(x) + \int \d \tilde{\pi}$
			\State Update $\mathcal{T}_\ell = \mathcal{T}_{\ell-1} \circ \left(\text{KR map of }\pibridgetilde{\ell}(x)\right)$
			\EndWhile
			\State Set $L=\ell$
			\State \Return $\mathcal{T}_L$
		\end{algorithmic}
	\end{algorithm}
	This procedure is summarized in Algorithm~\ref{algo:learning_from_data}.

	\section{Numerical Examples}\label{sec:Numerical_examples}
	We demonstrate the method presented in this paper with numerical examples, implemented in Julia.
	The implementation of our method can be found at~\url{https://github.com/benjione/SequentialMeasureTransport.jl}.
	For solving each SDP problems $\min_{A\succeq0} \widehat{\mathrm{D}}_\alpha(\overline{\pi}^{(\ell)}||g_A)$,
	we use the \textit{JuMP.jl} optimization library~\cite{Lubin2023} together with \textit{Hypatia.jl}, as an interior point solver~\cite{coey2022solving}.
	For more details of how we use SDP, see Appendix~\ref{app:Impl_details}.
	
	\subsection{Multimodal density from data by diffusion process}
	We demonstrate the diffusion-based model \eqref{eq:density_diffusion} by estimating a bimodal density, given by
	\begin{align*}
		\pi_X(\bm x) \quad \text{with } \bm x \sim \begin{cases}
			\mathcal{N}\left(\left(\begin{array}{cc}
				2 & 2
			\end{array}\right)^\top, \text{diag}\left(\begin{array}{cc}
				0.1 & 0.5
			\end{array}\right)\right)  \qquad \text{with probability } 0.5 \\
			\mathcal{N}\left(\left(\begin{array}{cc}
				-2 & -2
			\end{array}\right)^\top, \text{diag}\left(\begin{array}{cc}
				0.5 & 0.1
			\end{array}\right)\right)  \qquad \text{else.} 
		\end{cases}
	\end{align*}
	The density is learned from $1000$ independent samples together with $L=20$ bridging densities.
	The SoS functions we employ for estimating the densities use Legendre polynomials up to order $4$.
	\begin{figure}
		\centering
		\begin{subfigure}{0.49\textwidth}
			\includegraphics[width=\textwidth]{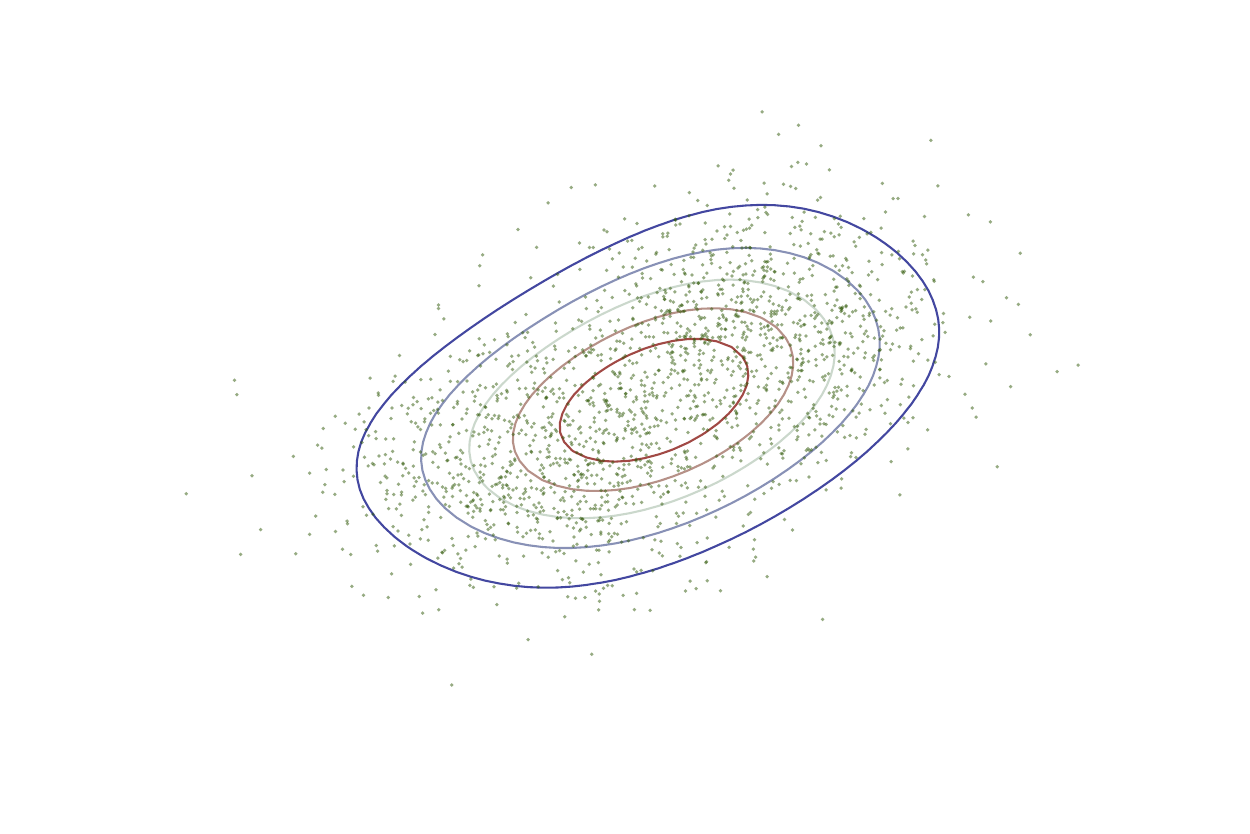}
			\caption{$t_1 \approx 1.0 $}
		\end{subfigure}
		\begin{subfigure}{0.49\textwidth}
			\includegraphics[width=\textwidth]{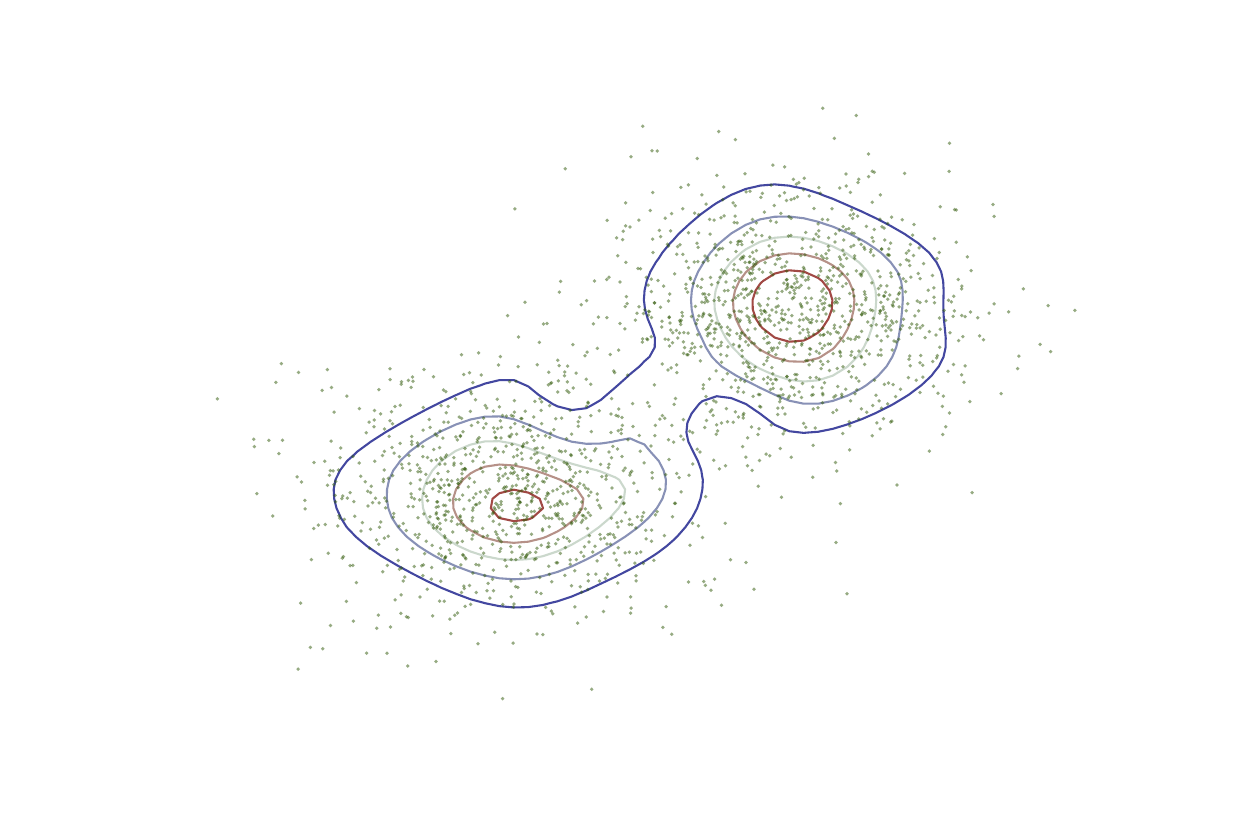}
			\caption{$t_2 \approx 0.71$}
		\end{subfigure}
		\begin{subfigure}{0.49\textwidth}
			\includegraphics[width=\textwidth]{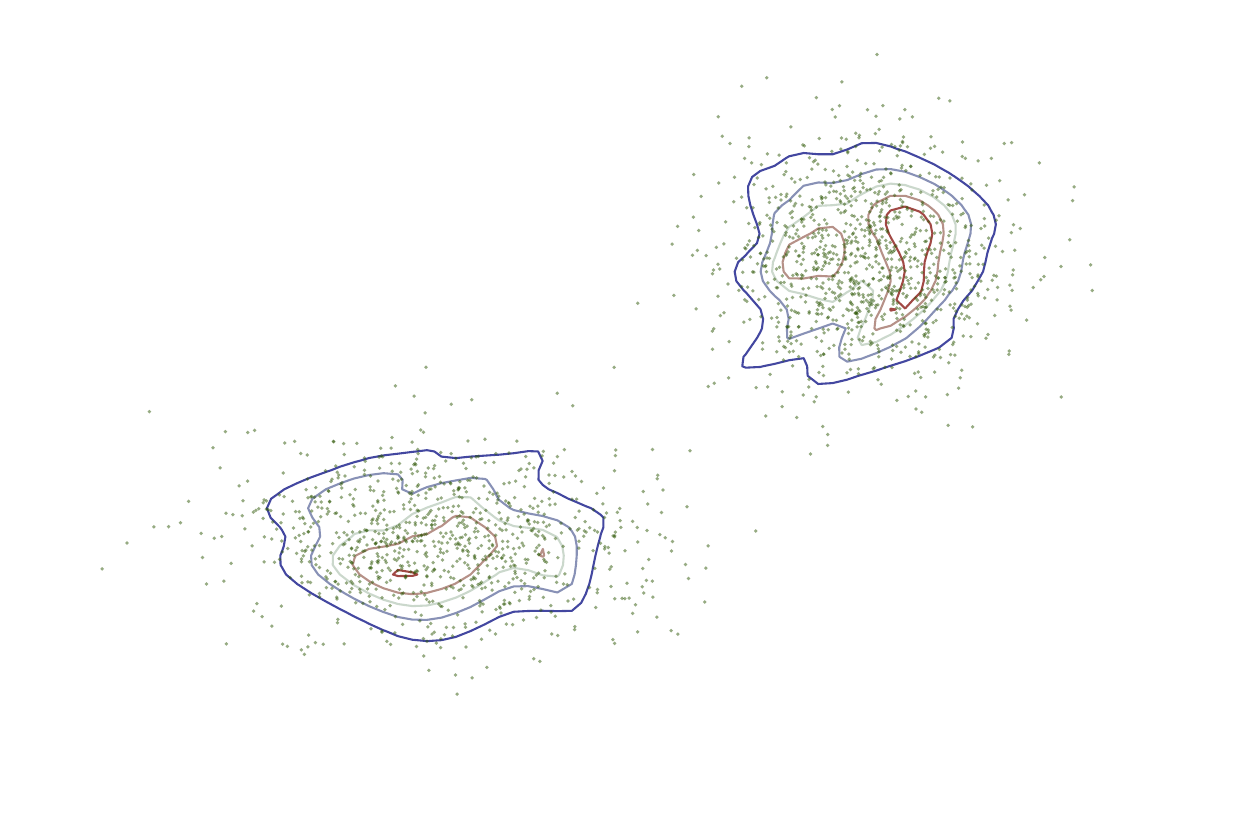}
			\caption{$t_{6} \approx 0.28 $}
		\end{subfigure}
		\begin{subfigure}{0.49\textwidth}
			\includegraphics[width=\textwidth]{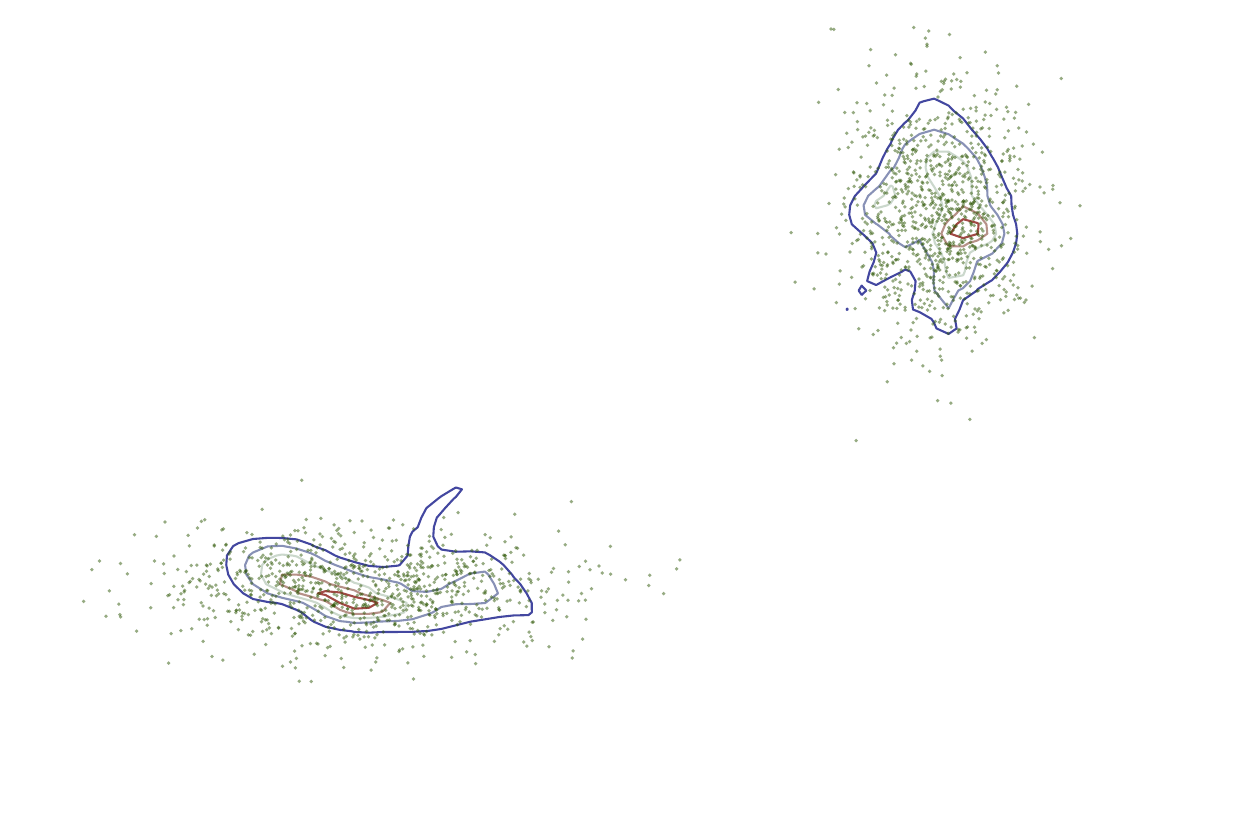}
			\caption{$t_{20} = 0$}
		\end{subfigure}
		\caption{Visualization of the reverse diffusion process at 4 different time steps, where the samples evolved after the given times are visualized as green dots and the estimated density as a contour.}
		\label{fig:multimodal_ML}
	\end{figure}
	In Figure~\ref{fig:multimodal_ML}, we provide a little cartoon of diffused samples and its density estimation at different timesteps in the process.
	In the first timestep, the density is simple enough that our model is able to capture it well.
	Following the reverse process, the previous estimations are used as preconditioners in order to learn the more and more complicated densities in the reverse diffusion process.
	
\subsection{Density from datasets}
We estimate the density of UCI datasets using Algorithm~\ref{algo:learning_from_data} and Lazy maps from Section~\ref{subsec:towards-high-dim}.
For the setup, we follow approaches from~\cite{baptista2023representation} and~\cite{uria_rnade_2014}.
We remove all discrete coefficients and coefficients with Pearson correlation higher than $0.98$.
We then randomly split the data into a training and test dataset, where $90\ \%$ are for training and $10\ \%$ for testing.
We evaluate the negative log likelihood function of the test dataset to determine how well the density fits the data.
We do this a total of $10$ times with different splits in train and test data for a better comparison.
This setup is identical with the one in~\cite{baptista2023representation} and we compare the results of the negative log likelihood directly.
These results are for the adaptive transport map (ATM) algorithm~\cite{baptista2023representation} and a multivariate Gaussian estimation.

ATM estimates a triangular map $\mathcal{T}:\R^d\rightarrow\R^d$ by parametrizing each of its components as $\mathcal{T}_i(x_{1},\hdots,x_i) = f_i(\bm x_{<i},0)+\int_{t=0}^{x_i} g(\partial_i f_i(\bm x_{<i},t)) \d t$ where $g(\xi)=\log(1+\exp(\xi))$ is the softplus function, where $f_i:\R^{i}\rightarrow\R$ is a polynomial function built in an adaptive (downward closed) polynomial set.

For our method, we use Algorithm~\ref{algo:learning_from_data} and apply density estimations with randomly picked $8$-dimensional Lazy maps.
For each of the Lazy maps, we choose polynomials of order up to $3$ so that $A \in \R^{70 \times 70}$.
\begin{table}%[htbp]
	\caption{Comparison of negative log likelihood function for different UCI datasets and density estimation methods. The best results are highlighted in bold.}\label{table:comparison_log_likelihood_UCI_datasets}
	\centering
	\begin{tabular}{p{2.0cm}|| c |c|c|c}
		\hline
		Dataset & (d, N) & SoS & Gaussian & ATM \\ %& \# sequential SoS \\
		\hline
		White wine & $(11, 4898)$ &  $\mathbf{11.2 \pm 0.2}$ & $13.2 \pm 0.5$ & $\mathbf{11.0 \pm 0.2}$ \\ %& $55.6 \pm 8.5$ \\
		Red wine & $(11, 1599)$ & $10.4 \pm 0.3$ & $13.2 \pm 0.3$ &  $\mathbf{9.8 \pm 0.4}$ \\ %& $30.2 \pm 5.4$ \\
		Parkinsons & $(15, 5875)$ & $4.7 \pm 0.2$ & $10.8 \pm 0.4$ & $\mathbf{2.8 \pm 0.4}$ \\ %& $67.8 \pm 3.9$ \\
		Boston & $(10, 506)$ & $6.6 \pm 0.5$ & $11.3 \pm 0.5$ & $\mathbf{3.1 \pm 0.6}$ \\ %& $23.9 \pm 2.2$ \\
		\hline
	\end{tabular}
\end{table}
The results are depicted in Table~\ref{table:comparison_log_likelihood_UCI_datasets}.
Note that for this table, the negative log likelihood function is calculated on the normalized data.

Next, we demonstrate conditional density estimation from data. Again, we use the exact same setup as in~\cite{baptista2023representation} and compare our results directly.
The estimated conditional density is $\pi(x_d | \bm x_{< d})$, where $d$ is the dimension of the dataset.
To evaluate the estimation, we use the negative conditional log likelihood, $- \tiny \frac{1}{N}\sum_{i=1}^{N} \ln\left(\widetilde{\pi}^{(L)}(x_d^i | \bm x_{< d}^i)\right)$ with $\bm x$ being the test data from the datasets.
Additionally, we consider an enriched dataset for the intermediate problems where, instead of \eqref{eq:blurred_samples}, we use the samples
\begin{equation}\label{eq:blurred_samples_K}
	\bm X^{(i,k)}_\ell = \exp(-t_\ell)  \bm X^{(i)} + \sqrt{1-\exp(-2 t_\ell)}  \bm Z^{(i,k)} ,
\end{equation}
where $Z^{(i,k)}, (i,k)\in\{1,\hdots,N\}\times\{1,\hdots,K\}$ for are independently sampled from $\mathcal{N}(0,I_d)$.
The effect of enriching this dataset for the yacht dataset is shown in Figure~\ref{fig:yachtconditionallikelihoodpaths} (a).

\begin{figure}%[htbp]
	\centering
	\begin{subfigure}{0.48\textwidth}
		\includegraphics[width=\textwidth]{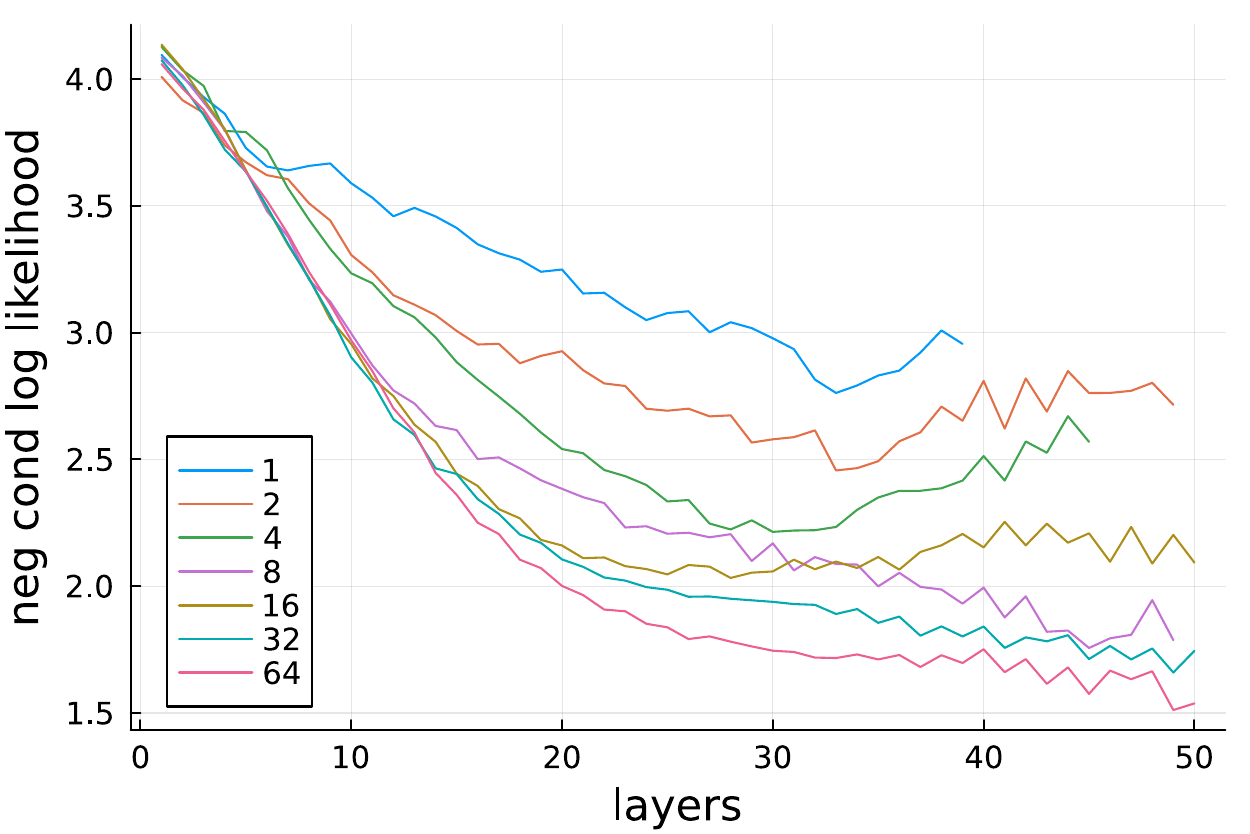}
		\caption{Conditional log likelihood paths for different amount of enrichment.}
	\end{subfigure}%
	\begin{subfigure}{0.48\textwidth}
		\includegraphics[width=\textwidth]{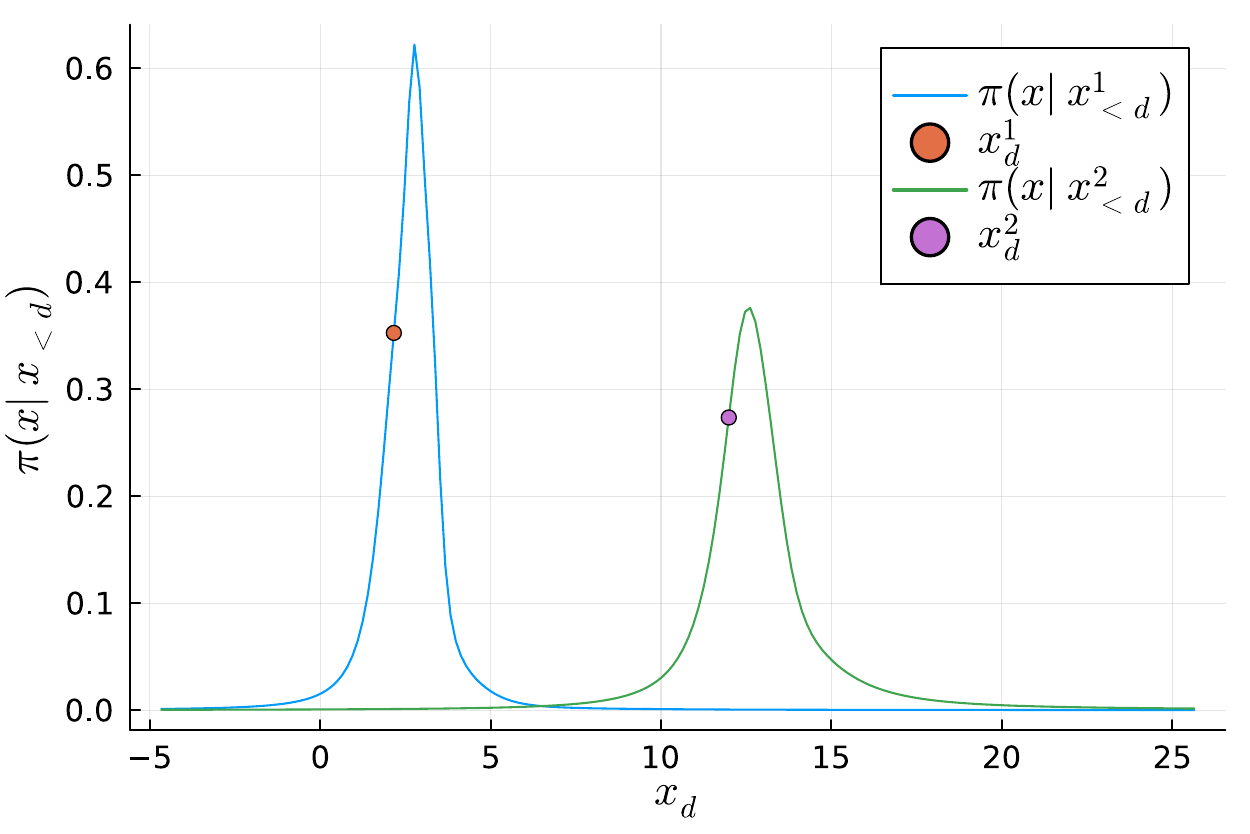}
		\caption{Estimated conditional densities of the yacht dataset evaluated at two test points.}
	\end{subfigure}
	\caption{(a) Conditional negative log likelihood of the validation set of the yacht dataset with different amount $K\in\{1,2,4,8,16,32,64\}$ of enrichment of the dataset \eqref{eq:blurred_samples_K}. (b) Conditional densities for two test points $x^1$ and $x^2$.}
	\label{fig:yachtconditionallikelihoodpaths}
\end{figure}

Once again, we compare our method with ATM~\cite{baptista2023representation} and an implementation of conditional normalizing flow, whose results we also take from~\cite{baptista2023representation}.
We compare with SoS using no enrichment (third column) and with an enrichment of 4 paths per training point (fourth column).
The results are shown in Table~\ref{table:comparison_cond_log_likelihood_UCI_datasets}.
\begin{table}
	\caption{Comparison of conditional negative log likelihood function for different UCI datasets and density estimation methods. The best results are highlighted in bold.}\label{table:comparison_cond_log_likelihood_UCI_datasets}
	\centering
	\begin{tabular}{p{1.5cm}|| c |c|c|c|c}
		\hline
		Dataset & (d, N) & SoS {\tiny(1 sample)} & SoS {\tiny(4 samples)} & ATM & NF \\%& \# seq. SoS {\tiny(4 samples)} \\
		\hline
		Boston & (12, 506) & $2.8 \pm 0.2$ & $\mathbf{2.5 \pm 0.2}$ & $\mathbf{2.6 \pm 0.2}$ & $\mathbf{2.4 \pm 0.1}$ \\%& $27.4 \pm 4.2$ \\
		Concrete & (9, 1030) & $3.4 \pm 0.3$ & $\mathbf{3.1 \pm 0.1}$ & $\mathbf{3.1 \pm 0.1}$ & $\mathbf{3.2 \pm 0.2}$ \\%& $44.0 \pm 4.8$ \\
		Energy & (10, 768) & $2.2 \pm 0.2$ & $\mathbf{1.7 \pm 0.2}$ & $\mathbf{1.5 \pm 0.1}$ & $\mathbf{1.7 \pm 0.3}$ \\ %& $34.8 \pm 6.1$  \\
		Yacht & (7,308) & $3.4 \pm 1.1$ & $2.0 \pm 0.6$ & $\mathbf{0.5 \pm 0.2}$ & $1.3 \pm 0.5$\\ %& $40.8 \pm 12.7$
		\hline
	\end{tabular}
\end{table}
SoS with enrichment archives comparable scores to ATM and NF.

In Figure~\ref{fig:yachtconditionallikelihoodpaths} (b) the estimated conditional densities evaluated at two validation points is depicted, which shows that the estimated densities fit the validation points well.

	\subsection{Susceptible-Infected-Removed (SIR) model}
	We consider an example from epidemiology and follow the setting in~\cite{cui_self-reinforced_2023}.
	Having access to the SIR model, a model describing the spread of a disease, see~\cite{kermack_contribution_1927}, we want to calibrate parameters of the model given observations on the amount of infected persons at different points in time.
	The SIR model is given by
	\begin{alignat*}{2}
		& \dot{S} = -\beta I S \\
		& \dot{I} = \beta S I - \gamma I  \\
		& \dot{R} = \gamma I.
	\end{alignat*}
	The parameters to be determined are $\gamma$ and $\beta$.
	The ODE is simulated for $t \in [0, 5]$ with $6$ equidistant observations $y_{j}$ in $I$ perturbed with noise $\epsilon_{j} \sim \mathcal{N}(0, 1)$, so that $y_{j} = I(t_j) + \epsilon_{j}$.
	We pose the calibration problem in an Bayesian setting, where the prior belief on the parameters is uniform in $[0,2]$ for all parameters and the likelihood function is gaussian and given by
	\begin{align*}
		\mathcal{L}(x|y) \propto \exp\left(\frac{- \sum_{j=1}^{6} (I(\frac{5j}{6}; x) - y_{j})^2}{2}\right) ,
	\end{align*}
	We approximate the posterior using $L=4$ layers with downward closed tensorized polynomials in $\Phi$ of maximum degree $6$.
	We choose tempered bridging densities as in Eq.~\eqref{eq:tempered} with $\beta_1 = 1/8$, $\beta_2 = 1/4$, $\beta_3 = 1/2$ and $\beta_4 = 1$ and use $1000$ evaluations of the unnormalized posterior for each layer.
	\begin{figure}[htbp]
		\centering
		\begin{subfigure}{0.49\textwidth}
			\includegraphics[width=\textwidth]{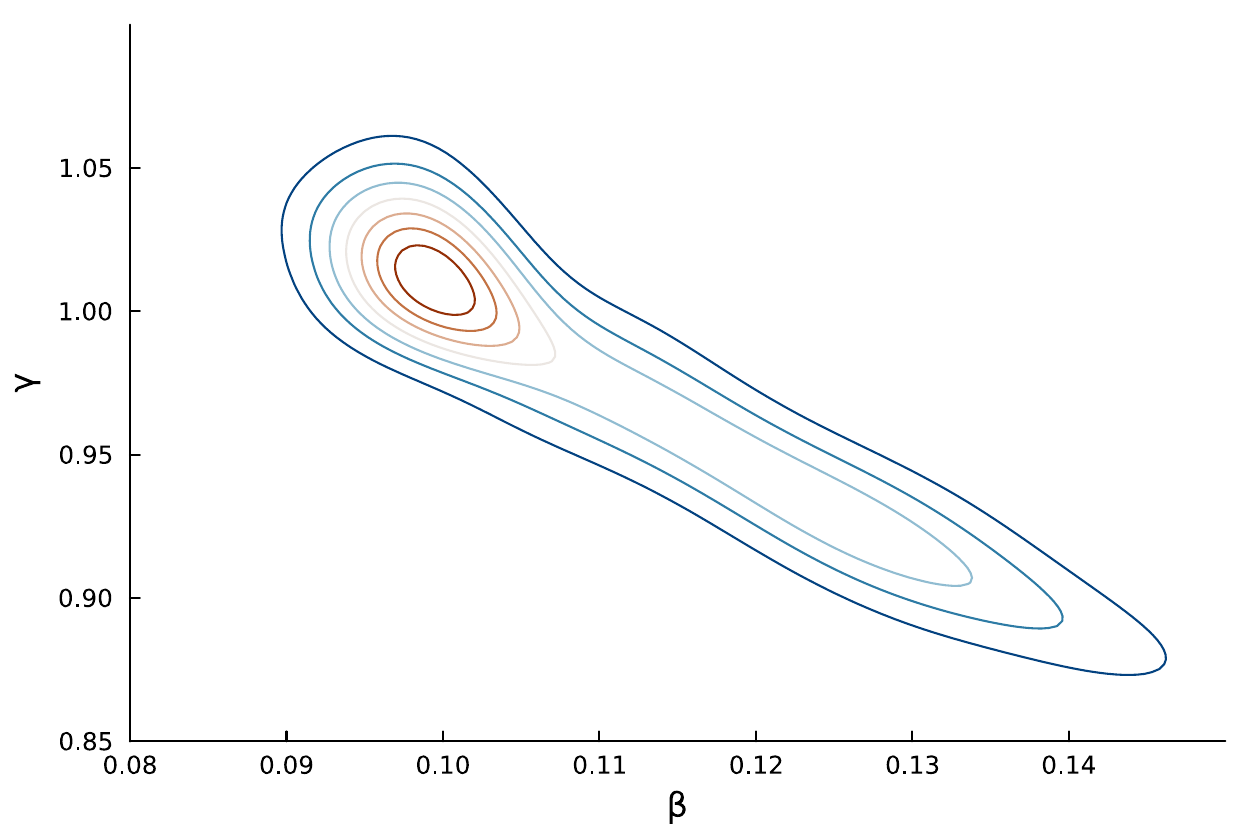}
			\caption{Variational posterior ($\alpha = 2$)}
		\end{subfigure}
		\begin{subfigure}{0.49\textwidth}
			\includegraphics[width=\textwidth]{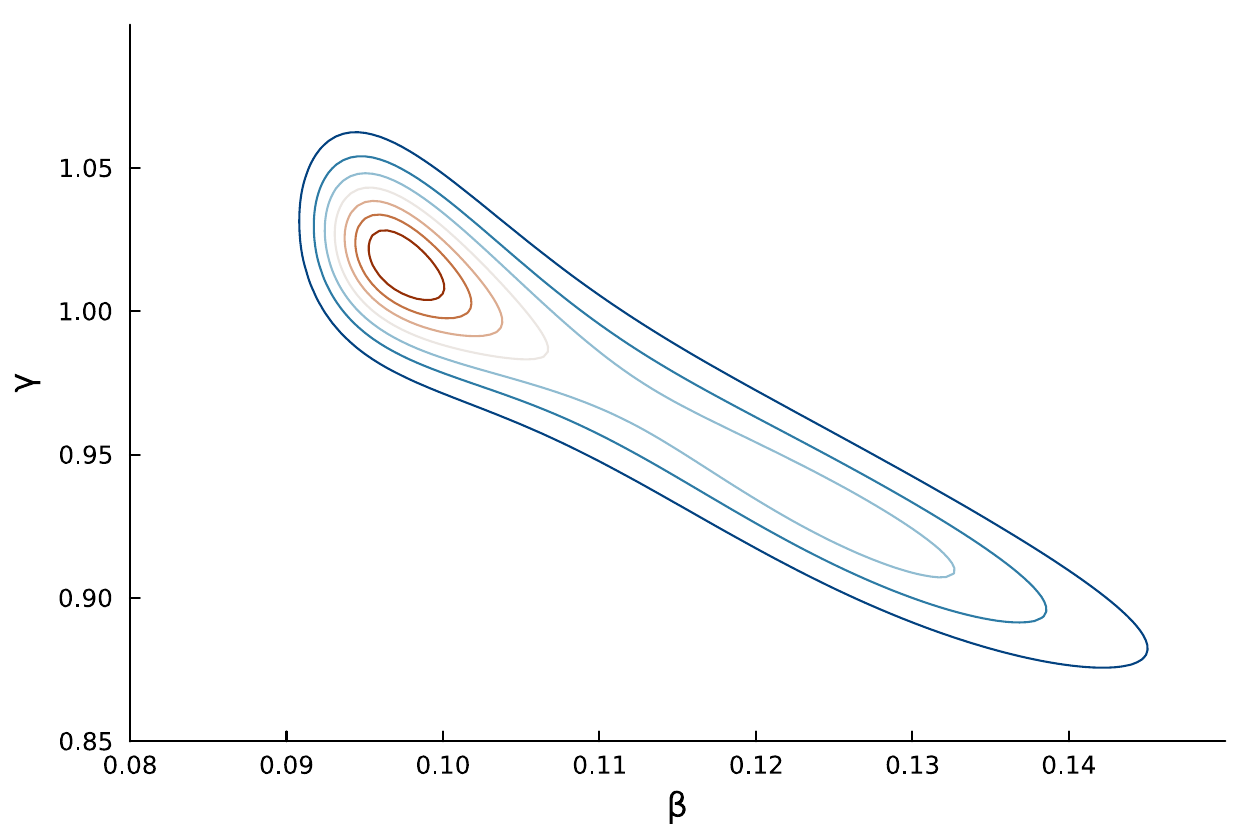}
			\caption{Unnormalized posterior}
		\end{subfigure}
		\caption{Comparison of the variational density (a) and the unnormalized posterior density (b) for the parameter space around the center of mass, for $\gamma \in [0.85, 1.1]$ and $\beta \in [0.08, 0.15]$.}
		\label{fig:SIR_comparison_variational_unnormalized}
	\end{figure}
	A comparison between the contours of the variational posterior and unnormalized posterior around the center of mass is depicted in Figure~\ref{fig:SIR_comparison_variational_unnormalized}.
	Visually, the density estimation fits the true density well. To quantify our results, we consider the estimated density for the task of self-normalized importance sampling and calculate the \textit{effective sampling size} (ESS).
	To do this, we run the density estimation $10$ times with drawing new noise $\epsilon_{j}$ in every iteration and calculate the ESS for density estimation using $\alpha$-divergences between $0.4$ and $3.0$.
	The results are shown in Figure~\ref{fig:ESS_of_alpha_divs}.
	\begin{figure}[htbp]
		\centering
		\includegraphics[width=0.5\textwidth]{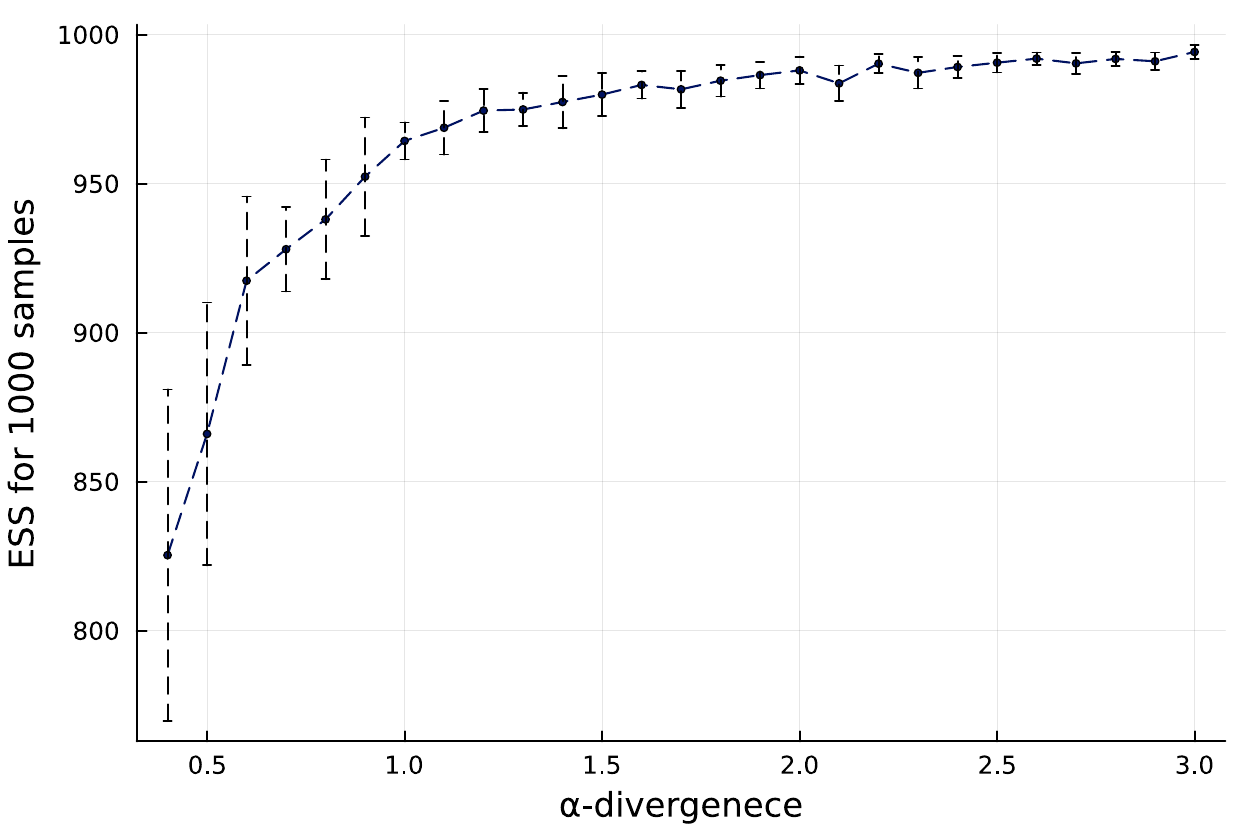}
		\caption{ESS for $\alpha$-divergence between $\alpha=0.4$ and $\alpha=3.0$ for $10$ runs on the SIR model with different realizations of noise. The ESS is given as effective samples for $1000$ samples.}\label{fig:ESS_of_alpha_divs}
	\end{figure}
	As expected, for higher $\alpha$ a better ESS is achieved.
	This is because the ESS can be written as $\text{ESS} = 1/(1 + \Var_{\pitartilde}\left[Z_{f}\right])$ and minimizing the $\chi^2$-divergence is equivalent to minimizing $\Var_{\pitartilde}\left[Z_{f}\right]$ (e.g. see~\cite{minka_divergence_2005}).
%	Since higher $\alpha$-divergences are upper bounds on the $\chi^2$-divergence, it is expected that minimizing over them performs similarly.
	The strong improvement in ESS for $\alpha \geq 1$ can be explained by the zero-avoiding property of the $\alpha$-divergences, which is a requirement for successful importance sampling.

	\section{Conclusion and Outlook}

	We introduced SoS functions for density approximation utilizable with a wide range of statistical divergences.
	We embedded this tool into the framework of measure transport~\cite{marzouk_introduction_2016} and used it in the context of sequential transport, as in~\cite{cui_scalable_2023}.
	The SoS functions we proposed are a generalization of~\cite{cui_self-reinforced_2023} and~\cite{westermann_measure_2023} with an extension to general divergences.
	To the best of our knowledge, our work is the first in measure transport to propose a linear function space in $\R_+^n$ to be utilizable with any convex statistical divergence.
	
	We provided a convergence analysis for sequential density estimation based on information geometry and compared it to existing results for the Hellinger distance in~\cite{cui_scalable_2023}.
	We demonstrated the applicability of our method with numerical examples.

	One of the remaining challenges to improve the scalability of our method is the sampling procedure for the SoS density estimation.
	It is well established that the sample-based approximation of a loss function has a strong influence on the approximation result. For instance, optimal sampling strategies have been proposed for the discretization of least squares problems, see~\cite{cohen_optimal_2017,dolbeault2022optimal}.
	Such strategy has been used from density estimation in the squared Hellinger distance in \cite{westermann_measure_2023,cui_self-reinforced_2023}.
	Adapting such strategies for the discretization of $\alpha$-divergences could greatly improve the efficiency of the proposed method.
	
	Another open topic is the choice of bridging densities.
	In this paper, we employ tempered densities for Bayesian inference (or diffusion-based densities when working with data) because it is the most convenient choice in practice. Recently, diffusion-based densities have been used in~\cite{grenioux2024stochastic,akhound2024iterated} to accelerate MCMC sampling, with an extra computational cost for evaluating the bridiging densities.
	A natural question, left for future work, is whether there exists an optimal choice of bridging densities.
	
	In our numerical experiments, we used Lazy maps with random subspaces to deal with problems of dimension greater than 8.
	More advanced techniques can be employed here, such as tensor-train \cite{cui_deep_2021} or gradient-based dimension reduction techniques \cite{li2024principal}.

	\backmatter
	
%	\bmhead{Supplementary information}
%	
%	If your article has accompanying supplementary file/s please state so here. 
%	
%	Authors reporting data from electrophoretic gels and blots should supply the full unprocessed scans for key as part of their Supplementary information. This may be requested by the editorial team/s if it is missing.
%	
%	Please refer to Journal-level guidance for any specific requirements.
	
\bmhead{Acknowledgements}
We would like to thank Ricardo Baptista for helpful discussions and comparisons between ATM and SoS, and Marylou Gabrié for great discussions on sampling with diffusion models.
The authors thank the IDEX for funding BZ a mobility grant and the partial support offered by the School of Mathematics and Statistics during his visit at the University of Sydney.
TC is supported by the ARC Discovery Project DP210103092. 
OZ acknowledges support from the ANR JCJC project MODENA (ANR-21-CE46-0006-01).
	
% 	\section*{Declarations}
%
% 	Some journals require declarations to be submitted in a standardised format. Please check the Instructions for Authors of the journal to which you are submitting to see if you need to complete this section. If yes, your manuscript must contain the following sections under the heading `Declarations':
%
% 	\begin{itemize}
% 		\item Funding
% 		\item Conflict of interest/Competing interests (check journal-specific guidelines for which heading to use)
% 		\item Ethics approval and consent to participate
% 		\item Consent for publication
% 		\item Data availability
% 		\item Materials availability
% 		\item Code availability
% 		\item Author contribution
% 	\end{itemize}
%
% 	\noindent
% 	If any of the sections are not relevant to your manuscript, please include the heading and write `Not applicable' for that section.
%
% 	%%===================================================%%
% 	%% For presentation purpose, we have included        %%
% 	%% \bigskip command. Please ignore this.             %%
% 	%%===================================================%%
% 	\bigskip
% 	\begin{flushleft}%
% 		Editorial Policies for:
%
% 		\bigskip\noindent
% 		Springer journals and proceedings: \url{https://www.springer.com/gp/editorial-policies}
%
% 		\bigskip\noindent
% 		Nature Portfolio journals: \url{https://www.nature.com/nature-research/editorial-policies}
%
% 		\bigskip\noindent
% 		\textit{Scientific Reports}: \url{https://www.nature.com/srep/journal-policies/editorial-policies}
%
% 		\bigskip\noindent
% 		BMC journals: \url{https://www.biomedcentral.com/getpublished/editorial-policies}
% 	\end{flushleft}
	
	\begin{appendices}

		\section{Proofs}

		\subsection{Proof of Equation \eqref{eq:AlphaDivNormalized}}
		\label{proof:measure_prob_bound_alpha_div}
		By definition~\eqref{eq:def_alpha_divergences} of the $\alpha$-divergence, we can write for $\alpha\notin\{0;1\}$
		\begin{align*}
			\AlphaDiv\left(f || g\right)
% 			&= \int \phi_\alpha\left(\frac{f}{g}\right) \d g \\
			&= \int \frac{Z_f^\alpha}{Z_g^\alpha} \frac{\left(\frac{\pi_f}{\pi_g}\right)^\alpha - \left(\frac{Z_f}{Z_g}\right)^{\alpha} - 1 + 1}{\alpha (\alpha - 1)} Z_g \d \pi_g + \int \frac{Z_f}{Z_g} \frac{\frac{\pi_f}{\pi_g} - \frac{Z_f}{Z_g} - 1 + 1}{\alpha - 1} Z_g \d \pi_g \\
			&= \frac{Z_f^{\alpha}}{Z_g^{\alpha - 1}} \AlphaDiv\left(\pi_f || \pi_g\right) + Z_g \left(\frac{\frac{Z_f^\alpha}{Z_g^\alpha} - 1}{\alpha (\alpha - 1)} + \frac{\frac{Z_f}{Z_g} - 1}{\alpha - 1}\right) \\
			&= \frac{Z_f^{\alpha}}{Z_g^{\alpha - 1}} \AlphaDiv\left(\pi_f || \pi_g\right) + Z_g \phi_{\alpha}\left(\frac{Z_f}{Z_g}\right).
		\end{align*}
		Similar derivations yield \eqref{eq:AlphaDivNormalized} for $\alpha\in\{0;1\}$.

		\subsection{Proof of Equation \eqref{eq:TransportProperty}}\label{proof:pushforward_pullback_equiv_in_alpha_div}
		Recall the formula $\mathcal{T}_\sharp g (\bm x)= g\circ \mathcal{T}^{-1}(\bm x) |\nabla \mathcal{T}^{-1}(\bm x)|$ for any unnormalized density $g$. Thus, the change of variable $x=\mathcal{T}(\bm y)$ permits to write
		\begin{align*}
			\AlphaDiv\left(f || \mathcal{T}_\sharp g \right)
% 			&= \int \phi_{\alpha}\left(\frac{f(x)}{\mathcal{T}_\sharp g(x)}\right)  \mathcal{T}_\sharp g(x) \d x \\
			&= \int \phi_{\alpha}\left(\frac{f(\bm x)}{g\circ \mathcal{T}^{-1}(\bm x) |\nabla \mathcal{T}^{-1}(\bm x)|}\right)  g\circ \mathcal{T}^{-1}(\bm x) |\nabla \mathcal{T}^{-1}(\bm x)| \d \bm x \\
			&= \int \phi_{\alpha}\left(\frac{f\circ\mathcal{T}(\bm y)|\nabla \mathcal{T}(\bm y)|}{g(\bm y) }\right)  g(\bm y) \d \bm y \\
			&= \AlphaDiv\left(\mathcal{T}^\sharp f ||  g \right),
		\end{align*}
		which is \eqref{eq:TransportProperty}.

		\subsection{Proof of Proposition~\ref{propo:SoS_marginalization}}\label{proof:marginalization_SoS}
		This proof is inspired by~\cite{cui_self-reinforced_2023} where the marginalization is done for squared polynomials only.
		We can write
		\begin{align*}
			\mathcal{L}g_A(\bm x)&=\int \Phi(\bm x_{-\ell},x_\ell')^\top A \Phi(\bm x_{-\ell},x_\ell') \d\mu_\ell(x_\ell') \\
			&= \sum_{\alpha, \beta \in\mathcal{K}} A_{\sigma(\alpha) \sigma(\beta)} \left(\prod_{i\neq\ell} \phi_{\alpha_i}^i(x_i)\right) \left(\prod_{i\neq\ell} \phi_{\beta_i}^i(x_i)\right) \int \phi_{\alpha_l}(x_l') \phi_{\beta_l}(x_l') \d\mu_\ell(x_l') \\
			&= \sum_{\alpha, \beta\in\mathcal{K}} \hat{\Phi}_{\sigma(\beta)} A_{\sigma(\alpha) \sigma(\beta)} M_{\sigma(\alpha) \sigma(\beta)} \hat{\Phi}_{\sigma(\alpha)} \\
			&= \hat{\Phi}(\bm x_{-\ell})^\top \left(A \odot M\right) \hat{\Phi}(\bm x_{-\ell})
		\end{align*}
		where
		\begin{align*}
			M_{\sigma(\alpha) \sigma(\beta)} = \int \phi_{\alpha_l}(x_l') \phi_{\beta_l}(x_l') \d\mu_\ell(x_l')  ,
		\end{align*}
		and $\hat{\Phi}_{\sigma(\alpha)}(\bm x_{-\ell}) = \prod_{i\neq \ell} \phi_{\alpha_i}^i(x_i)$.
		Note that the vector $\hat{\Phi}(\bm x_{-\ell})$ has size $m=|\mathcal{K}|$ and contains duplicated entries of the vector $\Phi_{-\ell}(\bm x_{-\ell})$ defined by $\hat{\Phi}_{\sigma_{-\ell}(\alpha_{-\ell})}(\bm x_{-\ell}) = \prod_{i\neq \ell} \phi_{\alpha_i}^i(x_i)$.
		By definition \eqref{eq:defP} of the matrix $P\in\R^{|\mathcal{K}_{-\ell}|\times m}$, we have $\hat{\Phi}_{\sigma(\alpha)}(x_{-\ell}) = P^\top \Phi_{-\ell}(x_{-\ell})$. This shows $M^{-\ell} = P(A \odot M)P^\top$ and concludes the proof.

		\subsection{Proof of Proposition \ref{prop:scheduler_tempering}}\label{proof:scheduler_tempering}
Recall the forumla \eqref{eq:def_alpha_divergences} for the $\alpha$-divergence
$$
 D_\alpha(f||g) = \frac{1}{\alpha(\alpha-1)} \left( \int (f/g)^\alpha g - g  \d \bm x\right) - \frac{\int f-g \d \bm x }{\alpha-1} ,
$$
so that for $\pi_{\beta}(\bm x)=\mathcal{L}(\bm x)^\beta\pi_0(\bm x)$ we have
\begin{align*}
 D_\alpha( \pi_{\beta+\Delta\beta} || \pi_{\beta} )
 &= \frac{1}{\alpha(\alpha-1)} \left( \int \mathcal{L}^{\alpha\Delta\beta}  -1 \d \pi_\beta  \right) - \frac{\int \mathcal{L}^{\Delta\beta} -1 \d \pi_\beta }{\alpha-1} \\
%  &= \frac{ (\int \mathcal{L}^{\alpha\Delta\beta} -1 \d \pi_\beta  )  -\alpha(\int \mathcal{L}^{\Delta\beta} -1\d \pi_\beta)}{\alpha(\alpha-1)}\\
 &= \frac{ \int (\mathcal{L}^{\alpha\Delta\beta}   -\alpha  \mathcal{L}^{\Delta\beta}) \d \pi_\beta + (\alpha-1)\int \d \pi_\beta }{\alpha(\alpha-1)} .
\end{align*}
Using the Tayor expansions
\begin{align*}
 \mathcal{L}^{\alpha\Delta\beta}
 &= 1+\alpha\Delta\beta\log(\mathcal{L})+\frac{(\alpha\Delta\beta\log(\mathcal{L}))^2}{2} + \mathcal{O}(\Delta \beta^3) \\
 \alpha\mathcal{L}^{\Delta\beta}
 &= \alpha+\alpha\Delta\beta\log(\mathcal{L})+\frac{\alpha(\Delta\beta\log(\mathcal{L}))^2}{2} + \mathcal{O}(\Delta \beta^3) ,
\end{align*}
we get
\begin{align}
 D_\alpha( \pi_{\beta+\Delta\beta} || \pi_{\beta} )
 &= \frac{1}{\alpha(\alpha-1)} \int \frac{(\alpha\Delta\beta\log(\mathcal{L}))^2}{2} -\frac{\alpha(\Delta\beta\log(\mathcal{L}))^2}{2} \d \pi_\beta  + \mathcal{O}(\Delta \beta^3) \nonumber\\
%  &= \frac{1}{2(\alpha-1)} \int \alpha(\Delta\beta\log(\mathcal{L}))^2 -(\Delta\beta\log(\mathcal{L}))^2 \d \pi_\beta  + \mathcal{O}(\Delta \beta^3) \\
 &= \frac{\Delta\beta^2}{2} \int \log(\mathcal{L})^2  \d \pi_\beta  + \mathcal{O}(\Delta \beta^3) \nonumber\\
 &= \frac{\Delta\beta^2}{2} C''(\beta)  + \mathcal{O}(\Delta \beta^3) , \label{eq:tmp207895}
\end{align}
where we used the fact that the normalizing constant $C(\beta) = \int \mathcal{L}^\beta \d\pi_0$ is such that $C''(\beta)= \int \mathcal{L}^\beta \log(\mathcal{L})^2 \d\pi_0 = \int \log(\mathcal{L})^2 \d \pi_\beta$.
Replacing $\beta$ with $\beta_\ell$ and $\Delta\beta$ with
$$
 \Delta\beta_\ell = \beta_{\ell+1}-\beta_\ell \overset{\eqref{eq:scheduler_beta}}{=} \frac{\beta'(\ell/L)}{L} + \mathcal{O}\left(\frac{1}{L^2}\right) ,
$$
we have that \eqref{eq:tmp207895} becomes $D_\alpha( \pi_{\beta_{\ell+1}} || \pi_{\beta_\ell} ) = \frac{\beta'(\ell/L)^2 }{2L^2} C''(\beta_\ell)  + \mathcal{O}(1/L^3)$. Then, the choice \eqref{eq:scheduler_tempering_unnormalized} yields \eqref{eq:control_tempering_unnormalized}.

\subsection{Proof of Proposition \ref{prop:scheduler_diffusion}}\label{proof:scheduler_diffusion}

The change of variable $\bm z=\frac{\bm x-e^{-\Delta t}y}{\sqrt{1-e^{-2\Delta t}}}$ permits to write
\begin{align*}
 \pi_{t+\Delta t}(\bm x)
 &= \int \kappa_{t+\Delta t}(\bm x,\bm y)\pi(\bm y) \d \bm y = \int \kappa_{\Delta t}(\bm x,\bm y)\pi_t(\bm y) \d \bm y \\
 &= \frac{1}{(2\pi(1-e^{-2\Delta t}))^{d/2}} \int \exp\left( -\frac{\|\bm x-e^{-\Delta t}\bm y\|^2}{2(1-e^{-2\Delta t})}\right)\pi_t(\bm y) \d \bm y \\
 &= \frac{1}{(2\pi(1-e^{-2\Delta t}))^{d/2}} \int \exp\left( -\frac{\|\bm z\|^2}{2}\right)\pi_t( e^{\Delta t}(x-\sqrt{1-e^{-2\Delta t}}\bm z) ) (e^{\Delta t}\sqrt{1-e^{-2\Delta t}})^d \d \bm z \\
 &= \frac{e^{d\Delta t}}{(2\pi)^{d/2}} \int \exp\left( -\frac{\|\bm z\|^2}{2}\right)\pi_t( e^{\Delta t}(\bm x-\sqrt{1-e^{-2\Delta t}}\bm z) )  \d \bm z \\
 &= e^{d\Delta t}\E\left[ \pi_t \left( e^{\Delta t}(\bm x-\sqrt{1-e^{-2\Delta t}}Z)  \right)   \right]
\end{align*}
Using the Taylor expansions
\begin{align*}
 e^{d\Delta t} &= 1+ d\Delta t + \tfrac{d^2}{2}\Delta t^2 +\mathcal{O}(\Delta t^3) \\
 e^{\Delta t}(\bm x-\sqrt{1-e^{-2\Delta t}}\bm z)
 &= x \underbrace{ -\sqrt{2\Delta t}\bm z+\Delta t \bm x - \tfrac{\sqrt{2}}{2}\Delta t^{3/2}  \bm z + \tfrac{1}{2}\Delta t^2 \bm x - \tfrac{5\sqrt{2}}{24}\Delta t^{5/2} \bm z}_{=\bm u}+ \mathcal{O}(\Delta t^{3})
\end{align*}
we deduce
\begin{align*}
 &\E\left[ \pi_t\left(  \bm x + U + \mathcal{O}(\Delta t^{3}) \right)   \right] \\
 &= \E\left[ \pi_t(\bm x)+\nabla\pi_t^\top U + \frac{U^\top\nabla^2\pi_t U}{2} + \mathcal{O}(\Delta t^{3} )   \right] \\
 &=  \pi_t +\nabla\pi_t^\top \left( \Delta t \bm x + \tfrac{\Delta t^2}{2} \bm x \right) + \frac{2\Delta t \E[Z^\top \nabla^2\pi_t Z] + 2\Delta t^2 \E[Z^\top \nabla^2\pi_t Z] + \Delta t^2 \bm x^\top \nabla^2\pi_t \bm x}{2} + \mathcal{O}(\Delta t^{3} )    \\
 &=  \pi_t +\nabla\pi_t^\top \left( \Delta t \bm x + \tfrac{\Delta t^2}{2} \bm x \right) +  \Delta t \trace(\nabla^2\pi_t) + \Delta t^2 \trace(\nabla^2\pi_t) + \tfrac{\Delta t^2}{2} \bm x^\top \nabla^2\pi_t \bm x + \mathcal{O}(\Delta t^{3} )  \\
 &=  \pi_t
 + \Delta t \left(\nabla\pi_t^\top \bm x + \trace(\nabla^2\pi_t) \right)
 + \Delta t^2 \left( \tfrac{1}{2} \nabla\pi_t^\top \bm x +\trace(\nabla^2\pi_t) + \tfrac{1}{2} \bm x^\top \nabla^2\pi_t \bm x  \right)  + \mathcal{O}(\Delta t^{3})
\end{align*}
and then
\begin{align*}
 \pi_{t+\Delta t}
 &= \left(1+ d\Delta t + d^2\Delta t^2/2 +\mathcal{O}(\Delta t^3)\right) \E\left[ \pi_t\left(  x + U + \mathcal{O}(\Delta t^{3}) \right)   \right] \\
 &= \pi_t
 + \Delta t \left(d\pi_t+\nabla\pi_t^\top \bm x + \trace(\nabla^2\pi_t) \right) \\
 &+ \Delta t^2 \left( d(\nabla\pi_t^\top \bm x + \trace(\nabla^2\pi_t)) + \tfrac{d^2}{2}\pi_t+ \tfrac{1}{2} \nabla\pi_t^\top \bm x +\trace(\nabla^2\pi_t) + \tfrac{1}{2} \bm x^\top \nabla^2\pi_t \bm x)  \right) + \mathcal{O}(\Delta t^{3})\\
 &= \pi_t + \Delta t A +  \Delta t^2  B + \mathcal{O}\left(\Delta t^{3} \right)
\end{align*}
where
\begin{align*}
 A &= d\pi_t+\nabla\pi_t^\top \bm x + \trace(\nabla^2\pi_t)\\
 B &= (d+\tfrac{1}{2})\nabla\pi_t^\top \bm x + (d+1)\trace(\nabla^2\pi_t) + \tfrac{d^2}{2}\pi_t  +  \tfrac{1}{2} \bm x^\top \nabla^2\pi_t
\end{align*}
We deduce
\begin{align*}
 D_\alpha(\pi_{t+\Delta t}||\pi_{t})
 &= \frac{1}{\alpha(\alpha-1)} \left( \int \left(\frac{\pi_{t+\Delta t}}{\pi_{t}}\right)^\alpha \d \pi_{t} - 1  \right) \\
 &= \frac{1}{\alpha(\alpha-1)} \left( \int \left(1+\frac{ \Delta t A + \Delta t^2 B + \mathcal{O}(\Delta t^{3} ) }{\pi_{t}}\right)^\alpha \d \pi_{t} - 1  \right) \\
 &= \frac{1}{\alpha(\alpha-1)}  \int \alpha \left(\frac{ \Delta t A + \Delta t^2 B  }{\pi_{t}}\right) + \frac{\alpha(\alpha-1)}{2}\left(\frac{ \Delta t A  }{\pi_{t}}\right)^2 \d \pi_{t}   + \mathcal{O}(\Delta t^{3} ) \\
 &= \frac{1}{\alpha-1}  \int  \left( \Delta t A + \Delta t^2 B \right)\d \bm x +
 \frac{\Delta t^2}{2}  \int  \frac{ A^2  }{\pi_{t}^2}  \d \pi_{t}   + \mathcal{O}(\Delta t^{3} )
\end{align*}
Note that because $\nabla\pi_t(\bm x)\rightarrow0$ when $\|\bm x\|\rightarrow\infty$, we have
\begin{align*}
 \int \nabla^2\pi_t \d \bm x &= \int \nabla \cdot (\nabla\pi_t) \d \bm x= 0 \\
 \int  \nabla \pi_t^\top \bm x \d \bm x  &= \sum_{i=1}^d  \int  \partial_i \pi_t x_i \d \bm x =  \sum_{i=1}^d  - \int   \pi_t  \d \bm x = -d \\
 \int x^\top \nabla^2\pi_t \bm x \d \bm x &= \sum_{i,j=1}^d \int \partial_{ij} \pi_t x_i x_j \d \bm x
 = \sum_{i=1}^d \left( \int \partial_{ii} \pi_t x_i x_i \d \bm x + \sum_{j\neq i} \int \partial_{ij} \pi_t x_i x_j \d \bm x\right) \\
 &= \sum_{i=1}^d \left( -\int \partial_{i} \pi_t (2x_i) \d \bm x - \sum_{j\neq i} \int \partial_{i} \pi_t x_i  \d \bm x\right) \\
 &= \sum_{i=1}^d \left( 2\int  \pi_t \d \bm x + \sum_{j\neq i} \int  \pi_t  \d \bm x\right) = \sum_{i=1}^d 2+(d-1) = d(d+1).
\end{align*}
We deduce that
\begin{align*}
 \int A \d \bm x &= \int d\pi_t+\nabla\pi_t^\top \bm x + \trace(\nabla^2\pi_t) \d \bm x = d-d+0 = 0 \\
 \int B \d \bm x &= \int (d+\tfrac{1}{2})\nabla\pi_t^\top \bm x + (d+1)\trace(\nabla^2\pi_t) + \tfrac{d^2}{2}\pi_t  +  \tfrac{1}{2} \bm x^\top \nabla^2\pi_t \d \bm x \\
 &= (d+\tfrac{1}{2})(-d)+(d+1)(0)+\tfrac{d^2}{2}+\tfrac{1}{2}d(d+1)  = 0
\end{align*}
so that
\begin{align*}
 D_\alpha(\pi_{t+\Delta t}||\pi_{t})
 &= \frac{\Delta t^2}{2}  \int  \frac{ A^2  }{\pi_{t}^2}  \d \pi_{t}   + \mathcal{O}(\Delta t^{3} ) \\
 &= \frac{\Delta t^2}{2}  \int  \left( d+ \frac{\nabla\pi_t^\top x}{\pi_t} + \frac{\trace(\nabla^2\pi_t)}{\pi_t} \right)^2  \d \pi_{t}   + \mathcal{O}(\Delta t^{3} )\\
 &= \frac{\Delta t^2}{2}  \int  \left( d+\nabla\log\pi_t^\top x + \trace(\nabla^2\log\pi_t) + \|\nabla\log\pi_t\|^2 \right)^2  \d \pi_{t}   + \mathcal{O}(\Delta t^{3} ) \\
 &= \frac{\Delta t^2}{2}  D(t)^2  + \mathcal{O}(\Delta t^{3} ) ,
\end{align*}
where we used the definition \eqref{eq:D} of $D(t)$.
By letting $t=t_\ell$ and $\Delta t = t_{\ell+1}-t_\ell = t'(\ell/L)/L+\mathcal{O}(L^{-2})$ where $t'(u) = - \Omega D(t(u))^{-1/2}$, we obtain \eqref{eq:scheduler_diffusion}. This concludes the proof.

\section{Implementation details}\label{app:Impl_details}
We implement our method in Julia.
To optimize the SoS functions, we use the \textit{JuMP.jl} library~\cite{Lubin2023}, an interface for implementing SDP and other optimization problems. As a solver, we use \textit{Hypatia.jl}~\cite{coey2022solving}, an interior point solver.

Next, we explain in more details how the minimization of $\alpha$-divergences is performed using SDP.
First, we discretize $\alpha$-divergences as in Eq.~\eqref{eq:DalphaHat} where $g$ is an SoS function with $A \succeq 0$.
This is handeled in SDP by using a PSD cone for $A$.
To integrate $\int g / \alpha \d \mu$, we use that $\int g \d \mu = \trace(A)$.
The main difficulty remaining is to map $g(\bm x^i)^{1-\alpha}$ to an SDP problem.
To do so, we use an auxiliary variable $t_i$, so that $t_i$ upper or lower bounds $f(\bm x^i)^\alpha g(\bm x^i)^{1-\alpha}$ and minimize over this variable at the same time.

For this, we use the power cone for $\alpha \neq 0, 1$, given by
\begin{align}
	K_p = \left\{(x,y,z)\in \R^3 : x^p y^{1-p} \geq |z|, x\geq 0, y \geq 0\right\}.
\end{align}
First, consider $\alpha < 0$. We use another variable $t_i$, so that
$t_i \geq f(x^i)^{\alpha} g(x^i)^{1-\alpha}$ and write
\begin{align}
	t_i &\geq f(x^i)^{\alpha} g(x^i)^{1-\alpha} \\
	t_i f(x^i)^{-\alpha} & \geq g(x^i)^{1-\alpha} \\
	t_i^{\frac{1}{1-\alpha}} \left(f(x^i)^{-1}\right)^{\frac{\alpha}{1-\alpha}} &\geq g(x^i).
\end{align}
Since $f(x^i)^{-1}$ can be very high, for numerical reasons it is better to use
\begin{align}
	t_i^{\frac{1}{1-\alpha}}  &\geq f(x^i)^{\frac{\alpha}{1-\alpha}} g(x^i)
\end{align}
and map this to the powercone with $(t_i, 1, f(x^i)^{\frac{\alpha}{1-\alpha}} q(x^i)) \in K_{1/(1-\alpha)}$.

Next, we consider $0 < \alpha < 1$. Note, that $\alpha (\alpha - 1)$ is negative in this region, hence, we use a variable $t_i$, so that $t_i \leq f(x^i)^{\alpha} g(x^i)^{1-\alpha}$. This directly maps to the powercone, by using $(f(x^i), g(x^i), t_i) \in K_{\alpha}$.

Last, we consider $\alpha > 1$ and again we seek for $t_i \geq f(x^i)^{\alpha} g(x^i)^{1-\alpha}$m but we can not use the same formulation as for $\alpha < 0$, since in the powercone $p \in [0, 1]$. Hence, we formulate
\begin{align}
	t_i &\geq f(x^i)^{\alpha} g(x^i)^{1-\alpha} \\
	t_i g(x^i)^{\alpha-1} &\geq f(x^i)^{\alpha} \\
	t_i^{\frac{1}{\alpha}} g(x^i)^{1- \frac{1}{\alpha}} &\geq f(x^i).
\end{align}
This maps to the powercone as $(t_i, g(x^i), f(x^i)) \in K_{1/\alpha}$.

For $\alpha = 0,1$, the powercone does not work. Instead, the exponential cone could be used. However, \textit{JuMP} has the possibility of using the \textit{relative entropy cone}, which directly works with the \textit{Hypatia.jl} solver.
This cone is given by
\begin{align}
	K_{n} = \left\{(u,\bm v,\bm w)\in \R^{1+2n} | u \geq \sum_{i=1}^n \bm w_i \ln\left(\frac{\bm w_i}{\bm v_i}\right)\right\}.
\end{align}
Both the KL-divergence, as well as the reverse KL-divergence, can be implemented using this cone.

\end{appendices}
	
	
	\bibliography{sn-bibliography.bib}% common bib file

\end{document}